\documentclass{article}

\usepackage{microtype}
\usepackage{graphicx}
\usepackage{subfigure}
\usepackage{booktabs} % for professional tables
\usepackage{bbm}
\newcommand{\ind}{\mathbbm{1}}
\usepackage{algorithm}
\usepackage{algorithmic}
\usepackage[hidelinks,colorlinks=true,linkcolor=blue,citecolor=blue]{hyperref}

\usepackage[preprint,nonatbib]{neurips_2026}

\usepackage{amsmath}
\usepackage{placeins}
\usepackage{amssymb}
\usepackage{mathtools}
\usepackage{amsthm}
\usepackage{natbib}
\usepackage{color}
\usepackage{amsfonts}       
\usepackage{url}
\usepackage{wrapfig}
\usepackage{tikz}
\usepackage{footmisc}
\usepackage{enumerate}
\usepackage{bibentry}
\usepackage{thmtools, thm-restate}
\nobibliography*

\usepackage{physics}
\mathtoolsset{showonlyrefs}
\usepackage{etoolbox}
\usepackage{makecell}
\usepackage{enumerate}
\usepackage{enumitem}
\usepackage{pifont}% http://ctan.org/pkg/pifont
\usepackage{soul}
\usepackage{ulem}
\usepackage{cancel}
\theoremstyle{plain}
\newtheorem{theorem}{Theorem}%[section]

\newtheorem{lemma}{Lemma}

\theoremstyle{definition}

\newtheorem{assumption}{Assumption}[section]
\newtheorem{remark}{Remark}
\theoremstyle{plain}
\newcommand{\fS}{\mathcal{S}}
\newcommand{\fA}{\mathcal{A}}

\newcommand{\fT }{\mathcal{T}}

\newcommand{\R}[1][]{\mathbb{R}^{#1}}
\newcommand{\E}{\mathbb{E}}

\newcommand{\tref}[1]{\text{\ref{#1}}}

\newcommand{\0}{\mathbf{0}}

\newcommand{\diag}{\operatorname{diag}}
\newcommand{\TF}{\text{TF}}

\newcommand{\pmdp}{p_{\text{MDP}}}
\newcommand{\rmdp}{r_{\text{MDP}}}
\newcommand{\wK}{\widetilde K}
\newcommand{\wM}{\widehat M}
\newcommand{\wP}{\widehat P}

\newcommand{\wT}{\widehat{\fT}}
\newcommand{\softmax}{\mathrm{softmax}}

\allowdisplaybreaks

\title{Beyond Linear Attention: Softmax Transformers \\ Implement In-Context Reinforcement Learning}

\author{%
  Zixuan Xie \\
  University of Virginia\\
  \texttt{xie.zixuan@email.virginia.edu} \\
  \And
  Xinyu Liu \\
  University of Virginia\\
  \texttt{xinyuliu@virginia.edu} \\
  \And
  Claire Chen\\
  California Institute of Technology\\
  \texttt{clairechen@caltech.edu}
  \And
  Shuze Daniel Liu\\
  Purdue University\\
  \texttt{liu4869@purdue.edu}
  \And
  Rohan Chandra \\
  University of Virginia\\
  \texttt{rohanchandra@virginia.edu} \\
  \And
  Shangtong Zhang \\
  University of Virginia\\
  \texttt{shangtong@virginia.edu} \\
}

\usepackage[textsize=tiny]{todonotes}

\begin{document}
\maketitle
\begin{abstract}
In-context reinforcement learning (ICRL) studies agents that, after pretraining, adapt to new tasks by conditioning on additional context without parameter updates. 
Existing theoretical analyses of ICRL largely rely on linear attention,
which replaces the softmax function in the standard attention with an identity mapping.
This paper provides the first theoretical understanding of ICRL without making the unrealistic linear attention simplification.
In particular, we consider the standard 
softmax attention used in practice.
% In particular, we consider softmax attention, where the attention weights are defined via a kernel operator, covering a broad class of nonlinear operators beyond linear attention.
We show that, with certain parameters, the layerwise forward pass of a Transformer with such softmax attention is equivalent to iterative updates of a weighted softmax temporal difference (TD) learning algorithm.
% We show that, with certain parameters,
% the layer by layer forward pass of Transformers with such softmax attention is equivalent to step by step updates of a weighted softmax temporal difference (TD) learning algorithm.
% Particularly,
% we prove that after generalizing \sz{I am still not very happy with this word though I choosed it. ``generalize'' usually means you can reduce back to it but here we can't reduce to exact softmax. can you brainstorm with ChatGPT to find some other words better than ``generalize'' and ``replace''? another option is that we don't say much from the softmax perspective. instead, we say we generalize the linear attention to a softmax one that includes way more nonlinear operators.} the softmax function in attention to a kernel operation,
% with certain parameters,
% the layer by layer forward pass of Transformers with such softmax attention is equivalent to step by step updates of a weighted softmax temporal difference (TD) learning algorithm.
Here, weighted softmax TD is a new RL algorithm that performs policy evaluation in kernel space and adopts both linear TD and tabular TD as special cases.
We also prove that under a certain contraction condition,
the policy evaluation error decays as the number of layers grows,
with the identified parameters above.
Finally, we prove that those parameters are a global minimizer of a pretraining loss,
explaining their emergence in our numerical experiments.
% Existing provable analyses of in-context policy evaluation largely rely on linear-attention constructions, while standard Transformers are nonlinear due to softmax attention and residual connections.
% We introduce weighted softmax TD for in-context policy evaluation, which includes linear TD and tabular as special cases. 
% We show that a Transformer block can implement weighted softmax TD in its forward pass using a general kernelization function $\tilde h$, and that stacking layers yields depth-wise error reduction under a contraction condition.
% We further prove the emergence of weighted softmax TD under a pretraining objective. 
% Finally, we complement our theory with numerical experiments. 
\end{abstract}

% \begin{abstract}
% In-context reinforcement learning (ICRL) studies agents that, after pretraining, improve on new tasks by conditioning on a context sequence without parameter updates. Prior analyses focus on linear settings, while standard Transformers are nonlinear due to softmax attention and residual blocks.
% We investigate, for the first time, whether nonlinear Transformers can provably realize ICRL on-policy evaluation tasks. First, we establish convergence by showing that nonlinear attention implements an in-context softmax temporal difference (TD) update in the forward pass, and that stacking layers yields depth-wise error reduction.
% Second, we study emergence under reinforcement pretraining, deriving sufficient conditions under which global minimizers of the pretraining loss realize this mechanism. We support these claims with theoretical guarantees and controlled experiments.
% \end{abstract}

\section{Introduction}

Recent works on in-context reinforcement learning (ICRL, \citet{moeini2025survey}) suggest that an RL agent can adapt to a new task at inference time without updating its parameters. 
In this framework, an agent network is pretrained on a distribution of tasks and then deployed with fixed pretrained parameters. 
At test time, the agent conditions on additional context generated within the new task, such as an interaction history $\tau_t\doteq(S_0,A_0,R_1,\ldots,S_{t-1},A_{t-1},R_t)$, 
and outputs actions or value estimates \citep{wang2016learning,duan2016rl,laskin2023incontext}.
Such fixed-parameter adaptation has been observed in
algorithm distillation \citep{laskin2023incontext}, human-timescale adaptation \citep{team2023human}, long-context embodied agents \citep{elawady2024relic}, Transformer-based adaptive agents \citep{raparthy2024generalization,grigsby2024amago,grigsby2024amago2,song2025reward}, and continuous-control or locomotion agents \citep{liu2025locoformer, Beaussant2025scaling}. 
Since the pretrained parameters do not change at test time, any adjustment in actions or value estimates must be produced by the
forward pass using the new context. 
This motivates the question we study: what RL process is implemented by the fixed agent network when it uses the context for adaptation?

We study this question in the reinforcement-pretraining regime.
This regime differs from supervised pretraining, where the desired
update rule is already present in the training targets \citep{laskin2023incontext, shi2023cross,zisman2024emergence,dai2024incontext,kirsch2023towards, lee2023supervised,lin2024transformers}. 
In reinforcement pretraining, the agent is instead optimized through task-level RL objectives, such as producing good actions in control or accurate value estimates in policy evaluation  \citep{duan2016rl, wang2016learning, grigsby2024amago}.
Thus, any context-dependent update implemented in the forward pass should be understood as an emergent computation induced by the pretraining objective, rather than as a imitated update rule.

To make this question concrete, we follow \citet{wang2025ictd,wang2025emergence} and consider in-context policy evaluation. 
In each task, a fixed policy $\pi$ induces a value function $v_\pi$. 
The agent receives a query state $s$ and a context $\tau_t$ of transitions collected under $\pi$, and outputs an estimate of $v_\pi(s)$ using its fixed pretrained parameters. 
The same agent network must handle a family of evaluation tasks whose transition dynamics, reward functions, and policies may differ. Temporal difference learning (TD, \citet{sutton1988learning}) is a
standard algorithm for policy evaluation, making it a natural candidate for the RL process implemented in the forward pass. 
\citet{wang2025ictd} prove that, under a carefully designed parameterization, the layer-by-layer forward pass of linear
Transformers is equivalent to step-by-step TD updates. 
\citet{wang2025emergence} further prove that those parameters can attain global optimality for a reinforcement-pretraining objective.
However, these results largely rely on linear self-attention where the softmax function in attention is replaced by an identity mapping to simplify analysis.
In this work, we remove this simplification by studying the standard 
softmax attention \citep{vaswani2017} used in practice.
We ask whether comparable in-context policy evaluation behavior can still arise under this setting.

\begin{figure}
    \centering
    % Subfigure for the MSVE plot (Top)
    \begin{subfigure} % Adjusted width for vertical stacking
        \centering
        \includegraphics[width=0.45\textwidth]{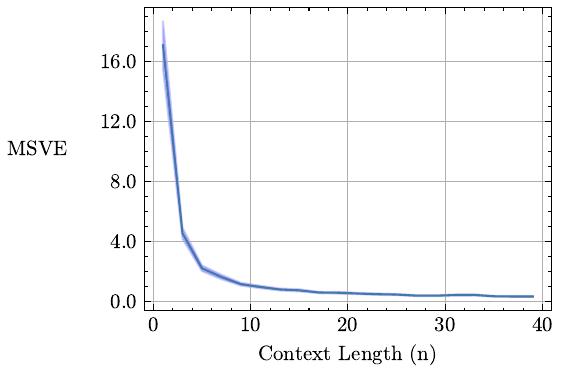}
        \caption{In-context policy evaluation with a 15-layer dual-head Transformer using softmax attention. 
        The model takes the context $\tau_t$ and a query state $s$ and outputs $\text{TF}_{\theta_*}(\tau_t, s)$ as the estimation of the state value $v_\pi(s)$.
        The y-axis is the mean square value error (MSVE) $\sum_s d_\pi(s)(\text{TF}_{\theta_*}(\tau_t, s) - v_\pi(s))^2$, with $d_\pi(s)$ being the stationary state distribution. 
        The curves are averaged over 300 randomly generated policy evaluation tasks, with shaded regions indicating standard errors. 
        Across tasks, the state space, transition dynamics, reward function, and policy vary, while a single parameterization $\theta_*$ is shared.}
        \label{fig:msve}
    \end{subfigure}
\end{figure}

% \sz{you need to have in mind about your story about where the parameters of this transformers come from, and make it explicit. You can say this is pretrained by some algorithm, then Q3 needs to be changed to ``why this pretraining gives rise to these parameters''. Or you can say you designed this parameters manually, which is less ideal but can still work. But you need to pick one and say it.} 
Figure~\ref{fig:msve} gives the motivating construction.
Let $\TF_{\theta_*}$ denote the dual-head Transformer with a manually constructed parameterization $\theta_*$ used in Figure~\tref{fig:msve}.
% The attention in $\TF_{\theta_*}$ instantiates softmax attention by using softmax weighting over similarity scores.
Given a context \(\tau_t\) and a query state \(s\), the model outputs \(\mathrm{TF}_{\theta_*}(\tau_t,s)\) as an estimate of \(v_\pi(s)\).
We report the value approximation error averaged over many tasks and policies with distinct value functions, using a single fixed $\theta_*$ throughout.
Notably, this error decreases as the context length $t$ grows.
% Specifically, with $\theta_*$ fixed across all tasks \sz{here you mentioned ``all tasks'' but really in your current writeup you didn't mention you have many tasks for in-context policy evaluation.}, 
% Notably, the value approximation error of $\TF_{\theta_*}$ decreases as the context length $t$ grows.
Since the target value functions vary widely across tasks, this improvement is unlikely to be explained by hard-coding $\theta_*$ to any particular true value function.
Instead, the model appears to leverage the context to refine its prediction by performing a policy-evaluation update in its forward pass.
This observation leads to three central questions.
\begin{enumerate}[label={(Q\arabic*)}]
\item What specific policy-evaluation algorithm is being implemented by $\TF_{\theta_*}$?\label{(Q1)}
\item What convergence guarantees hold as the number of layers grows?\label{(Q2)}
\item What kind of pretraining could give rise to these parameters?\label{(Q3)} 
% \sz{you can be more explicit by saying ``give rise to these parameters''}
\end{enumerate}
% This phenomenon leads to three central questions.
% \begin{enumerate}[label={(Q\arabic*)}]
% \item What is the specific policy evaluation algorithm being implemented by the Transformer's forward pass?\label{(Q1)}
% \item What convergence guarantees hold as the forward pass repeatedly applies this update?\label{(Q2)} 
% \item What kind of pretraining gives rise to this in-context learning capability?\label{(Q3)} 
% \end{enumerate}
In this paper, we address the above questions by providing the first theoretical analysis of in-context policy evaluation under standard softmax attention.
We show that the empirical pattern in Figure~\ref{fig:msve} can be explained by a novel algorithm, which we term \textbf{weighted softmax TD}.
We provide a constructive design of a dual-head Transformer layer whose forward pass implements a batch analogue of the nonparametric TD update \citep{berthier2022nonparametric} with softmax weighting.
% We provide a constructive design of a dual-head Transformer layer whose forward pass implements a batch analogue of the nonparametric TD update \citep{berthier2022nonparametric} in kernel space.
% The induced attention kernel aggregates TD errors over the context while holding the current value estimate fixed within the layer.
Moreover, we show that this dual-head layer admits an equivalent reparameterization as single-head attention composed with a variant of the Temporal Shift Module (TSM, \citet{lin2019tsm}), which is parameter-free.
This reparameterization reduces computation and further simplifies our inference-time convergence analysis.

% In this paper, we formalize the mechanism suggested by these experiments and give the first theoretical treatment of in-context policy evaluation in the nonlinear 
% \sz{always remember that linear attention is a nonlinear function} 
% setting.
% We show that the empirical pattern in Figure~\tref{fig:msve} is precisely explained by a novel algorithm, which we term \textbf{weighted softmax TD} 
% \sz{will we have multi-step version? if not, no need to mention one-step here.}.
% We provide a constructive design of a dual-head Transformer layer that exactly executes 
% \sz{``execute'' is a vague word, say explicitly what it does} an in-context 
% \sz{do we need the prefix in-context here? as we discussed, ``in-context'' is a way of implementation and is not part of an algorithm} 
% batch analogue of the nonparametric TD update of \citet{berthier2022nonparametric}.
% The induced attention kernel 
% \sz{it's funny becuase in your version you never mention attention kernel before. this is the first appearance of the word ``kernel''. so nobody can really understand what attention kernel is.} 
% aggregates TD errors over the context while holding the current value estimate fixed within the layer.
% Moreover, we show that this dual-head layer admits an equivalent
% reparameterization as a single head attention composed with a variant of the Temporal Shift Module (TSM, \citet{lin2019tsm}), which is itself parameter-free.
% This reparameterization reduces computation and further simplifies our convergence analysis.
Taken together, our contributions answer \ref{(Q1)}-\ref{(Q3)} in turn. 
For \ref{(Q1)}, we construct a dual-head Transformer layer that implements a weighted softmax TD update step in its forward pass. 
For \ref{(Q2)}, we prove an inference-time convergence guarantee as depth grows, under a suitable contraction condition. 
For \ref{(Q3)}, we show that the desired parameters are global minimizers of a pretraining loss, explaining their emergence in our experiments.
% Taken together, our contributions are threefold.
% \begin{enumerate}[label={(\arabic*)}]
% \item Answering \ref{(Q1)}, we show by construction that a dual-head Transformer layer can be configured to implement a weighted softmax TD update step in its forward pass.
% \item Answering \ref{(Q2)}, we prove an inference-time convergence guarantee as the number of layers increases under a suitable contraction condition.
% \item Answering \ref{(Q3)}, we show that the desired parameters are global minimizers of a pretraining loss, explaining their emergence in our numerical experiments.
% \end{enumerate}
% \begin{enumerate}[label={(\arabic*)}]
% \item Answering \ref{(Q1)}, we show by construction that a single layer of a nonlinear \sz{again nonlinear} dual-head Transformer can be configured to exactly implement one step of weighted softmax TD.
% \item Answering \ref{(Q2)}, we prove a depth-wise convergence guarantee as depth increases under a suitable contraction condition. \sz{emphasize that this is for inference time behavior}
% \item Answering \ref{(Q3)}, we analyze the emergence of softmax ICTD  under reinforcement pretraining and derive sufficient conditions under which this algorithmic structure is realized by global minimizers of the pretraining objective.
% \end{enumerate}

\section{Background}
\label{sec:background}
We begin by establishing the technical foundations for our work. All vectors are treated as column vectors. We denote the identity matrix in $\R[n\times n]$ by $I_n$.
We use $0_{m \times n}$ and $1_{m \times n}$ to denote zero matrix and all-ones matrix. 
Let $e_{i}$ be the standard basis vector for the $i$-th row.
The transpose of a matrix $Z$ is denoted by $Z^\top$. The inner product of two vectors $x$ and $y$ is denoted by either $\langle x, y \rangle$ or $x^\top y$.

% \subsection{Kernels as similarity-based weighting rules} 
% \label{sec:kernel}
% % \lsz{Why do you call it kernel? Do you think attention weights is a better name?}
% % \sz{you need to provide some citations for those new generation reviewers who don't necessarily know kernel and this kernel-weighted bootstrapping update.}
% We use the term \textit{kernel} to describe a similarity-based weighting rule \citep{nadaraya1964regression,watson1964smooth,scholkopf2002kernel,rasmussen2006gaussian}.
% Given reference elements $\{u_i\}_{i=1}^n$ and a query element $u_q$, a kernel $K$ assigns nonnegative weights $K(u_q,u_i)\ge 0$.
% Let $f$ denote a prediction function defined on these elements.
% Given bootstrapped targets $\{y_i\}_{i=1}^n$ and current predictions $\{f(u_i)\}_{i=1}^n$, define prediction errors
% $\epsilon_i \doteq y_i - f(u_i)$.
% A kernel-weighted bootstrapping update at the query $u_q$ takes the form \citep{sutton1988learning,berthier2022nonparametric}
% \begin{equation}
% \label{eq:kernel-bootstrap}
% f(u_q) \leftarrow f(u_q) + \alpha \sum_{i=1}^n K(u_q,u_i)\,\epsilon_i .
% \end{equation}

\subsection{Transformers and Softmax Self-Attention}
\label{sec:trans}
A Transformer's core innovation is the self-attention mechanism \citep{vaswani2017}. 
Given a prompt $Z \in \R[(d+3)\times(n+1)]$, a standard single-head 
self-attention layer processes the prompt via
\begin{equation}
\label{eq:attn}
\textstyle
\text{Attn}_{W_v, W_q, W_k}(Z) \doteq W_v Z\, \softmax\bigl(Z^\top W_k^\top W_q Z + M\bigr),
\end{equation}
where $W_v \in \R[(d+3)\times(d+3)]$, $W_k \in \R[m \times (d+3)]$, 
and $W_q \in \R[m \times (d+3)]$ are the value, key, and query weight 
matrices, and $\softmax$ is applied column-wise. 
The mask $M \in \R[(n+1)\times(n+1)]$ isolates the query column from the context columns
\begin{equation}
\label{eq:mask}
\textstyle
M[i,j] = \begin{cases} 0, & i \leq n, \\ -\infty, & i = n+1, \end{cases}
\end{equation}
so that the query column does not contribute as a source token. 
For analytical convenience, we follow \citet{cheng2024fgd} by reparameterizing $W_v$ as $V_l$ and denoting the query-key score matrix $W_k^\top W_q$ by $A_l$. 
An $L$-layer Transformer with parameters $\{(V_l, A_l)\}_{l=0,\ldots,L-1}$ 
updates the representation at layer $l+1$ recursively from an initial 
prompt $Z_0 \in \R[(d+3)\times(n+1)]$ as
\begin{equation}
\label{eq:repattn}
\textstyle
Z_{l+1} \doteq Z_l + V_l Z_l\, \softmax\bigl(Z_l^\top A_l Z_l + M\bigr).
\end{equation}
We define the attention matrix at layer $l$ by
\begin{equation}
\label{eq:kermatrix}
\textstyle
\wK_l \doteq \softmax\bigl(Z_l^\top A_l Z_l + M\bigr) \in \R[(n+1)\times(n+1)].
\end{equation}
Its entry $\wK_l[i,j]$ specifies the attention weight assigned to source token $i$ when aggregating into token $j$. By construction of $M$, $\wK_l[n+1, j] = 0$ for all $j$.
For notational simplicity, we suppress $M$ in the remainder of the paper, with the understanding that the mask in \eqref{eq:mask} is applied throughout.
Since policy evaluation requires a scalar prediction, we extract a single entry from the final representation matrix.
Following prior work \citep{ahn2023transformers, wang2025ictd}, we define
\begin{equation}
\label{eq:output}
    \text{TF}_L\left(Z_0; \{V_l, A_l\}_{l=0}^{L-1}\right) \doteq Z_L[d+3, n+1],
\end{equation}
to denote the output of the $L$-layer Transformer given an input $Z_0$.

For comparison, the existing ICRL theory of \citet{wang2025ictd, 
wang2025emergence} uses linear attention, replacing the masked softmax in \eqref{eq:repattn} with the identity map:
\begin{equation}
\label{eq:linear-attn-compare}
\textstyle
Z_{l+1}^{\mathrm{lin}} = Z_l + V_l Z_l\, (Z_l^\top A_l Z_l).
\end{equation}
In the broader in-context learning literature, \citet{cheng2024fgd} further replace the softmax with a general activation $h$, giving 
$Z_l + V_l Z_l\, h(Z_l^\top A_l Z_l)$. 
We instead retain the standard softmax \citep{vaswani2017}, bringing our analysis closer to the architecture used in practice.

\subsection{Reinforcement Learning and Policy Evaluation}

Since we focus on the policy evaluation problem with a fixed policy, we work directly with the induced Markov Reward Process (MRP, \citet{bellman1957markovian}). 
An MRP is a tuple $(\fS, p, r, \gamma, p_0)$ consisting of a finite state space $\fS$, a transition kernel $p: \fS \times \fS \to [0, 1]$, a reward function $r: \fS \to \R$, a discount factor $\gamma \in [0, 1)$, and an initial distribution $p_0$. 
Equivalently, an MRP arises from a Markov Decision Process $(\fS, \fA, \pmdp, \rmdp, \gamma, p_0)$ under a fixed policy $\pi: \fS \times \fA \to [0,1]$ by marginalizing over actions: $p(s'|s) = \sum_a \pi(a|s)\pmdp(s'|s,a)$ and $r(s) = \sum_a \pi(a|s)\rmdp(s,a)$.
A trajectory $(S_0, R_1, S_1, R_2, \dots)$ is sampled by drawing 
$S_0 \sim p_0$ and, at each step, $S_{t+1} \sim p(\cdot|S_t)$ with 
$R_{t+1} \doteq r(S_t)$. 
Define the state transition matrix $P_\pi \in \R[|\fS| \times |\fS|]$ by $P_\pi[s, s'] = p(s'|s)$, 
and assume the Markov chain induced by $P_\pi$ is ergodic on $\fS$ with unique stationary distribution $\mu_\pi: \fS \to [0,1]$. 
We aim to estimate the value function $v_\pi(s) \doteq \E\qty[\sum_{k=0}^{\infty} \gamma^k R_{t+k+1} \mid S_t = s]$.

One fundamental task in RL is policy evaluation, where the goal is estimating the value function $v_\pi$. It is well-known that $v_\pi$ is the unique fixed point of the Bellman expectation operator $\fT^\pi: \R[\abs{\fS}] \to \R[\abs{\fS}]$ defined as
\begin{equation}
\label{eq:expectedbellman}
    (\fT^\pi v)(s) \doteq \E_\pi[R_{t+1} + \gamma v(S_{t+1}) | S_t=s].
\end{equation}
Temporal-Difference learning (TD, \citet{sutton1988learning}) is among the most influential and widely used algorithms for policy evaluation.
% Define the TD error at a transition $(s,s')$ by
% \begin{equation}
% \label{eq:deltass}
%     \delta(s,s') \doteq r(s) + \gamma v(s') - v(s),
% \end{equation}
% where $r$ and the transition dynamics are induced by the MRP, and $v$ is instantiated as the current value estimate.
For large or continuous state spaces where tabular representations are 
impractical, a long line of work has combined RL with kernel methods to 
obtain non-parametric value function approximation 
\citep{ormoneit2002kernel, farahmand2016regularized, koppel2021policy, 
berthier2022nonparametric}. We follow the formulation of 
\citet{berthier2022nonparametric}. Let $\mathcal{X}$ be a measurable state 
space and $\kappa: \mathcal{X} \times \mathcal{X} \to \R$ a positive-definite 
kernel with associated reproducing kernel Hilbert space $\mathcal{H}$. 
Define the TD error at a transition $(x, x')$ by
\begin{equation}
\label{eq:deltass}
\delta(x, x') \doteq r(x) + \gamma v(x') - v(x).
\end{equation}
Given a sampled transition $(X_t, X_{t+1})$, non-parametric TD updates the value estimate $v_t \in \mathcal{H}$ at every query state $y \in \mathcal{X}$ via
\begin{equation}
\label{eq:nonparam-td-general}
v_{t+1}(y) = v_t(y) + \alpha_t \, \delta(X_t, X_{t+1}) \, \kappa(X_t, y),
\end{equation}
where $\{\alpha_t\}$ is a sequence of learning rates.
In this paper we work with the finite state space $\fS$. Restricting \eqref{eq:nonparam-td-general} to 
$\mathcal{X} = \fS$ and to a query state $S_n \in \fS$ yields the update
\begin{equation}
 \label{eq:nl td update}
v_{t+1}(S_n) = v_t(S_n) + \alpha_t \, \delta(S_t, S_{t+1}) \, \kappa(S_t, S_n),
\end{equation}
which is the form we build on in the rest of the paper.

\section{Transformers Can Implement In-Context Weighted Softmax TD}
\label{sec:algo}

We now answer \ref{(Q1)} by introducing a new RL algorithm, \textit{weighted softmax TD}, and constructing a dual-head Transformer that implements its update in the forward pass. Notably, this construction reveals the parameters we use to generate the results in Figure~\ref{fig:msve}.

\paragraph{Weighted softmax TD.} 
Given a trajectory $\tau_n=(S_0,R_1,\dots,S_n)$ sampled from an MRP, weighted softmax TD forms TD errors using \eqref{eq:deltass} with the sampled reward instantiated by $R_k$: 
\begin{equation}
\label{eq:tderror}
\delta_k^{(t)} \doteq R_k + \gamma v_t(S_k) - v_t(S_{k-1}),
\end{equation} 
where $k$ indexes transitions within $\tau_n$ and $t$ indexes the iteration. 
Non-parametric TD in \eqref{eq:nl td update} updates value estimates using a single transition. 
We introduce weighted softmax TD as a batch variant that aggregates $\delta_k^{(t)}$ over $\tau_n$ through softmax weights.
Specifically, for the trajectory states $\{S_j\}_{j=0}^{n}$, define the softmax weight assigned to the $k$-th transition by 
\begin{equation}
\label{eq:softmax-weight-sync}
\textstyle
K(S_{k-1},S_j)
\doteq
\frac{\exp(g(S_j,S_{k-1}))}{\sum_{m=1}^{n}\exp(g(S_j,S_{m-1}))},
\quad k=1,\dots,n,
\end{equation}
where $g:\mathcal{S}\times\mathcal{S}\to\mathbb{R}$ is a score function. 
Thus, the value estimates $\{v_t(S_j)\}_{j=0}^{n}$ are updated by
% Evaluated on the trajectory states $\{S_j\}_{j=0}^{n}$, the update aggregates $\delta_k^{(t)}$ over $\tau_n$ as
\begin{align}
\label{eq:wktd-update}
\tag{weighted softmax TD}
\textstyle
v_{t+1}(S_j)=  v_t(S_j)  +  \alpha_t \sum_{k=1}^{n}
    \delta_k^{(t)} K(S_{k-1}, S_j).
\end{align}

In particular, if $g(S_j,S_{k-1})=\log\kappa(S_j,S_{k-1})$ for some positive similarity function $\kappa$, then the weights reduce to 
$K(S_{k-1}, S_j) = \frac{\kappa(S_j,S_{k-1})}{\sum_m \kappa(S_j,S_{m-1})}$, recovering a normalized form of \eqref{eq:nl td update}.

We next show that the dual-head Transformer used in generating Figure~\ref{fig:msve} implements the \eqref{eq:wktd-update} in a single forward pass.
We begin by specifying the prompt construction and the layerwise update rule.
% We encode the trajectory $\tau_n$ as context and the query state $S_n$ into a prompt matrix $Z_0\in\R[(d+3)\times(n+1)]$.
% Using shorthands $x_k \doteq x(S_k)$ and $x_q \doteq x(S_n)$, we set
We encode the trajectory $\tau_n$ into a prompt matrix $Z_0\in\R[(d+3)\times(n+1)]$,
where the last column corresponds to the query state $S_n$.
Using shorthands $x_k \doteq x(S_k)$, we set
\begin{equation}
\label{eq:z_initial}
Z_0 \doteq \begin{bmatrix}
    x_0 & x_1 & \cdots & x_{n-1} & x_n \\
    R_1 & R_2 & \cdots & R_n & 0 \\
    0 & 0 & \cdots & 0 & 0 \\
    0 & 0 & \cdots & 0 & 0
\end{bmatrix}
=\begin{bmatrix}
    X\\
    R\\
    \0\\
    \0
\end{bmatrix}.
\end{equation}
Following \citet{wang2025ictd}, the last two rows serve as memory that the network writes and reuses across layers.

We consider an $L$-layer Transformer indexed by $l=0,1,\dots,L-1$.
Each layer consists of two attention heads, both taking $Z_l$ as input.
% We refer to row $d+2$ as the target-memory row and row $d+3$ as the current-memory row.
To state the layer update clearly, we define $\Delta Z_l= 0_{(d+3)\times(n+1)}$, and then construct two heads as follows.
\paragraph{Current-value head.}
Let $\qty{(V_l, A_l)}_{l=0,1,\dots,L-1}$ be the parameters of the current-value head in the $L$-layer Transformer.
It writes an additive update to row $d+3$ of $\Delta Z_l$:
\begin{equation}
\label{eq:tfupdate}
\Delta Z_l[d+3,:]
=
\big(V_l Z_l\,\softmax(Z_l^\top A_l Z_l)\big)[d+3,:].
\end{equation}
We choose $(V_l,A_l)$ as
\begin{equation}
\label{eq:para-now}
\textstyle
    V_l \doteq 
    \begin{bmatrix}
        0_{(d+2)\times d} & 0_{(d+2)\times 3} \\
        0_{1\times d} & \begin{bmatrix}1 & 1 & -1\end{bmatrix}
    \end{bmatrix},
    \,\,
    A_l \doteq 
    \begin{bmatrix}
        I_d & 0_{d\times 3} \\
        0_{3\times d} & 0_{3\times 3}
    \end{bmatrix}.
\end{equation}
By construction, the only non-zero row of $V_l$ is its last row, and hence \eqref{eq:tfupdate} modifies only the last row of $\Delta Z_l$.
With the choice of $A_l$ in \eqref{eq:para-now}, the score matrix $Z_l^\top A_l Z_l$ depends only on the first $d$ rows of $Z_l$, which by \eqref{eq:z_initial} store the fixed trajectory features $\qty{x_k}_{k=0}^n$ throughout all layers.
Hence the attention matrix $\wK_l$ in \eqref{eq:kermatrix} is layer-invariant, and we drop the layer index, writing $\wK$ for $wK_l$.
Moreover, $\widetilde{K}$ realizes the softmax weight $K$ in \eqref{eq:softmax-weight-sync} with the score function $g(S_j, S_{k-1}) = \langle x(S_j), x(S_{k-1}) \rangle$.
% Thus, under our softmax attention \eqref{eq:kermatrix}, the induced kernel matrix $\wK_l$ depends only on similarities between context features $\{x_k\}$ and the query feature $x_q$.
% In particular, $\wK_l$ is layer-invariant under our construction, and we may write $\wK_l \equiv \wK$ for all $l\ge 0$.
% Thus, under the softmax attention matrix \eqref{eq:kermatrix}, the induced kernel matrix $\wK_l$ depends only on similarities among the trajectory features $\{x_k\}_{k=0}^{n}$.
% In particular, $\wK_l$ is layer-invariant under our construction, and we may write $\wK_l \equiv \wK$ for all $l\ge 0$.
% Moreover, $\widetilde{K}$ coincides with the softmax weight $K$ in \eqref{eq:softmax-weight-sync} with score function $g(S_j, S_{k-1}) = \langle x(S_j), x(S_{k-1}) \rangle$.

\paragraph{Target-value head.}
The target-value head reuses the same parameters $(V_l,A_l)$ in \eqref{eq:para-now} and first forms the same kernel aggregation as in \eqref{eq:tfupdate}.
It then routes the aggregated quantities to the predecessor columns via a shift matrix
\begin{equation}
\label{eq:pi}
\textstyle
    \Pi \doteq 
\begin{bmatrix}
0 & 0 & 0 & \cdots & 0 \\
1 & 0 & 0 & \cdots & 0 \\
\vdots &  & \ddots & \ddots & \vdots \\
0 & \cdots & 1 & 0 & 0 \\
0 & \cdots & 0 & 1 & 0
\end{bmatrix}.
\end{equation}
This head writes an additive update to row $d+2$ of $\Delta Z_l$ as
\begin{equation}
\label{eq:tfupdate-prev}
\textstyle
\Delta Z_l[d+2,:]
=
\gamma\,
\big(V_l Z_l\, \softmax(Z_l^\top A_l Z_l)\Pi\big)[d+3,:].
\end{equation}

The resulting layerwise recursion sums the contributions of the two heads as
\begin{equation}
\label{eq:dualhead-out}
\textstyle
    Z_{l+1} \doteq Z_l + \Delta Z_l.
\end{equation}
Recall \eqref{eq:output}, we define the scalar output at layer $l$ as 
\begin{equation}
\label{eq:outputv}
\textstyle
    v_l(S_n) \doteq Z_l[d+3,n+1].
\end{equation}
The following theorem shows that $v_l(S_n)$ evolves according to the weighted softmax TD update.
\begin{theorem}
\label{thm:onestep}
Consider the $L$-layer dual-head Transformer with recursion \eqref{eq:dualhead-out} and parameters $\qty{(V_l, A_l)}_{l=0}^{L-1}$ given by \eqref{eq:para-now}, initialized at $Z_0$ in \eqref{eq:z_initial}.
% Let $v_l(S_n)$ be defined as in \eqref{eq:outputv}.
Then the forward pass updates $v_l(S_n)$ in \eqref{eq:outputv} as
\begin{equation}
\label{eq:update}
\textstyle
    v_{l+1}(S_n)
    =
    v_l(S_n)
    +
    \sum_{k=1}^{n}
    \delta_k^{(l)} K(S_{k-1}, S_n).
\end{equation}
\end{theorem}
The proof is in Section~\tref{proof:onestep}.
Since \eqref{eq:update} matches the \eqref{eq:wktd-update} update with $l$ as the iteration index, Theorem~\ref{thm:onestep} shows that the Transformer forward pass can implement iterations of the \eqref{eq:wktd-update} update at inference time.
We refer to this implementation paradigm as softmax in-context TD (softmax ICTD).
For simplicity, we set $\alpha_l = 1$ throughout.
The general case follows by replacing $V_l$ with $\alpha_l V_l$ in
\eqref{eq:repattn}, which rescales the update magnitude
but does not alter the structure of the construction in \eqref{eq:para-now}.
\FloatBarrier

\section{Dual-Head Attention = Single-Head Attention + Temporal Shift Module}
\label{sec:tsm-reparam}
The dual-head block in Section~\ref{sec:algo} is redundant and inconvenient for analysis.
Each layer computes two heads that share the same term $V_l Z_l \softmax(Z_l^\top A_l Z_l)$.
Specifically, \eqref{eq:tfupdate-prev} applies $\Pi$ to the same aggregated quantities in \eqref{eq:tfupdate} and scales them by $\gamma$ before writing.
Motivated by this observation, we reparameterize the block as a single-head attention followed by a parameter-free temporal shift module (TSM; \citet{lin2019tsm}).
% \sz{only here you need to name other elements of the last two rows. You have two options. You can name them now. But this is a claim because you are saying the two rows have certain relationship so you need to provide a link to a proof in appendix. Another option is that you don't name it and just compare how you generate the two rows (the two $\Delta Z_l$ eq) and then make the claim.}

Using $(V_l,A_l)$ from~\eqref{eq:para-now}, define
\begin{equation}
\label{eq:zhalf}
\textstyle
    Z_{l+\frac12} \doteq Z_l + V_l Z_l \softmax\bigl(Z_l^\top A_l Z_l\bigr),
\end{equation}
which matches the update in~\eqref{eq:tfupdate} on the current-memory row.
In the spirit of TSM, we introduce two fixed shift matrices
\begin{equation}
\label{eq:tsm}
\textstyle
    U \doteq I - e_{d+2}e_{d+2}^\top,\qquad
    W \doteq \gamma\, e_{d+2}e_{d+3}^\top.
\end{equation}
Recall from~\eqref{eq:pi} that $\Pi\in\R[(n+1)\times(n+1)]$ is the one-step predecessor shift matrix.
Therefore, $U$ clears the target-memory row, and $W(\cdot)\Pi$ routes the $\gamma$-scaled current-memory row to the predecessor columns of the target-memory row; see Figure~\ref{fig:tsm} for an illustration.
The resulting single-head+TSM update is
\begin{equation}
\label{eq:tsm-forward}
\textstyle
    Z_{l+1} = UZ_{l+\frac12} + WZ_{l+\frac12}\Pi.
\end{equation}
We next show that \eqref{eq:tsm-forward} is exactly equivalent to the original dual-head attention.

% Using $\qty{(V_l, A_l)}_{l=0,1,\dots,L-1}$ from~\eqref{eq:para-now}, define
% \begin{equation}
% \label{eq:zhalf}
%     Z_{l+\frac12} \doteq Z_l + V_l Z_l M \tilde h\big(Z_l^\top A_l Z_l\big),
% \end{equation}
% which is exactly the current-value head update written to the last row.
% Let $\Pi\in\mathbb{R}^{(n+1)\times(n+1)}$ be the one-step predecessor routing matrix defined in \eqref{eq:pi}.
% Define
% \begin{equation}
% \label{eq:tsm}
%     R \doteq I - e_{d+2}e_{d+2}^\top,\quad
%     W \doteq \gamma e_{d+2}e_{d+3}^\top.
% \end{equation}
% The single-head plus TSM update is then
% \begin{equation}
% \label{eq:tsm-forward}
%     Z_{l+1} = RZ_{l+\frac12} + WZ_{l+\frac12}\Pi.
% \end{equation}

\begin{lemma}
\label{lem:equiv}
Fix $Z_0$ defined in~\eqref{eq:z_initial}. Under the parameter correspondence
in~\eqref{eq:para-now}, for all $l=0,\ldots,L-1$ the dual-head block~\eqref{eq:dualhead-out}
and the single-head+TSM block~\eqref{eq:tsm-forward} map the same input $Z_l$ to the
same output $Z_{l+1}$.
\end{lemma}
The proof is in Section~\tref{proof:equiv}.
% Despite being equivalent, the following lemma shows single-head+TSM is more efficient.
% \begin{lemma}
% \label{lem:complexity}
% Let $p \doteq d+3$ and the prompt length be $n+1$.
% When $n\gg p$, dual-head attention costs $2\cdot O(n^2p)$ due to two attention passes, whereas single-head+TSM costs $O(n^2p)+O(np)$.
% \end{lemma}
% The proof is in Section~\tref{proof:complexity}.
In the rest of the paper, we analyze inference-time behavior using the single-head+TSM form.

\begin{figure}[htbp]\centering
  \begin{minipage}[t]{0.26\textwidth}\centering
    \includegraphics[width=\linewidth]{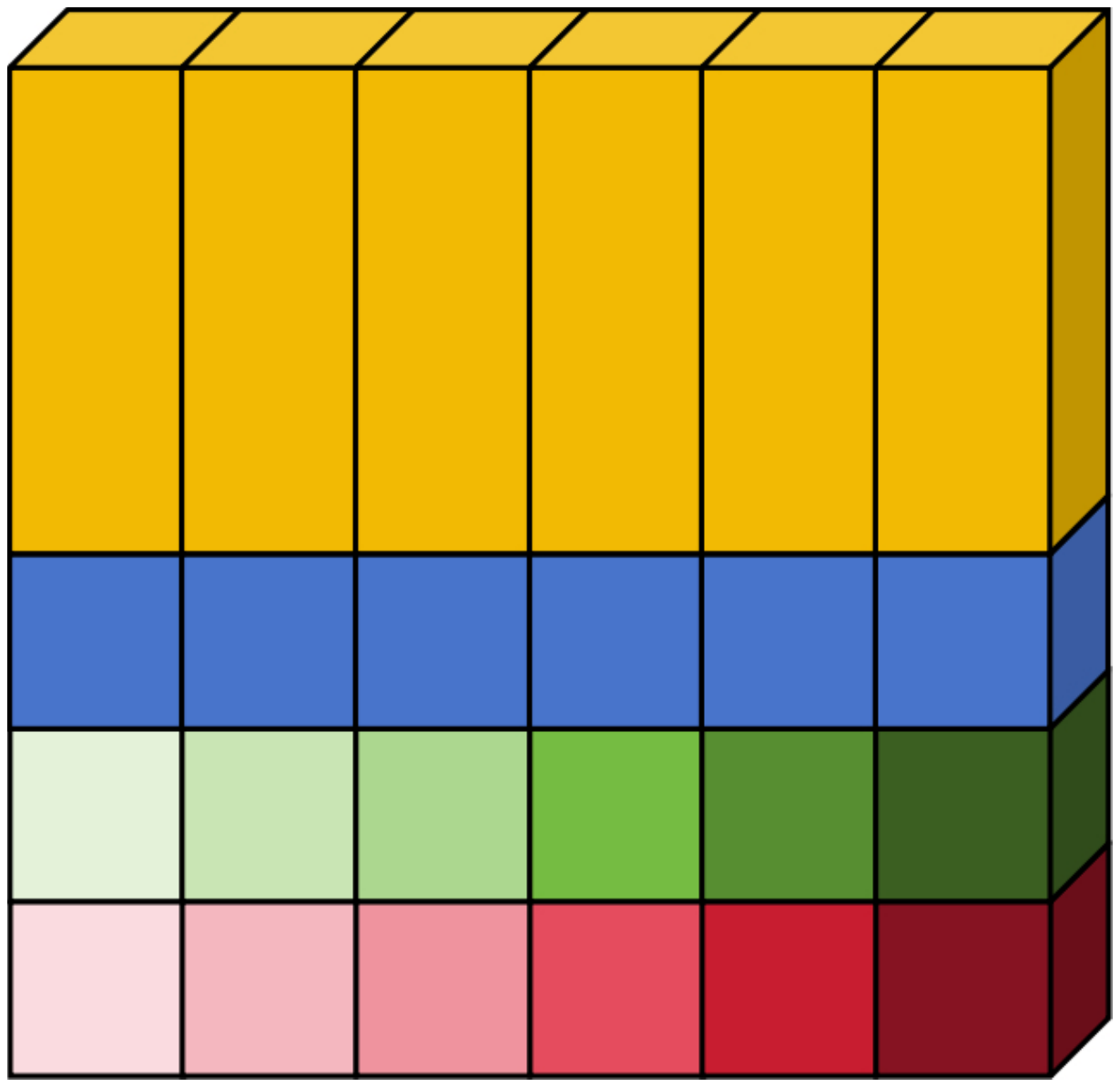}\\[2pt]
    \small (a) $Z_{l+\frac{1}{2}}$.
  \end{minipage}
  \hspace{0.01\textwidth}
  \begin{minipage}[t]{0.26\textwidth}\centering
    \includegraphics[width=\linewidth]{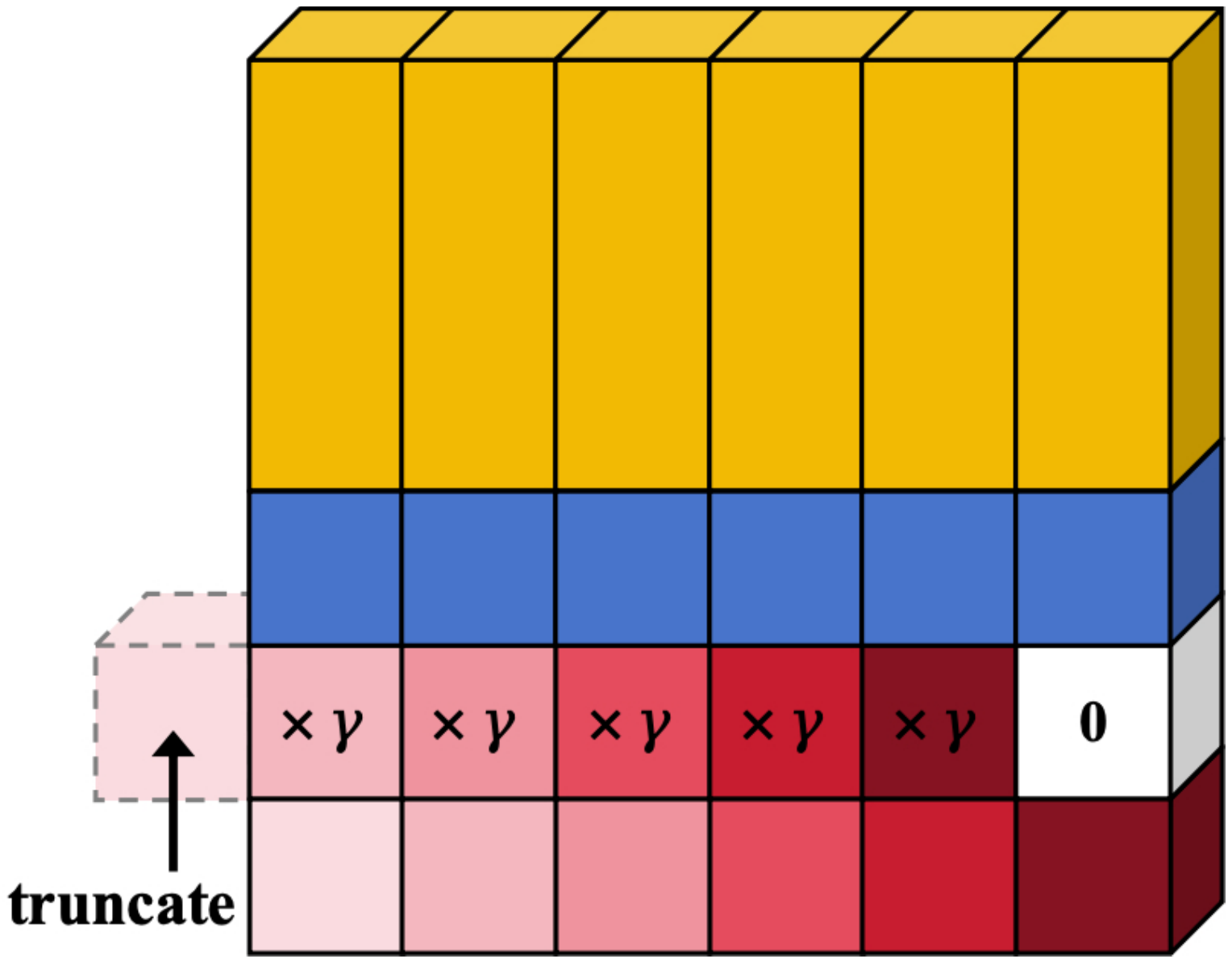}\\[2pt]
    \small (b) $Z_{l+1}$ (left-shift with $\times\gamma$, right pad 0).
  \end{minipage}
  \caption{Original vs. shifted memory rows.}
  \label{fig:tsm}
\end{figure}

\section{Inference-Time Convergence}
\label{sec:convergence}

We now address \ref{(Q2)} by establishing the convergence of softmax ICTD 
to the true value function $v_\pi$ as the number of layers $L \to \infty$ 
and the context length $n \to \infty$. Throughout the rest of the paper, 
we assume the trajectory $\tau_n=(S_0,R_1,\ldots,S_n)$ visits every state, 
i.e., $\{S_0, S_1, \ldots, S_{n-1}\} = \fS$.

To analyze the recursion \eqref{eq:update} on $\fS$, for a given context 
trajectory $\tau_n$, we define
\begin{equation}
\label{eq:Mn}
\textstyle
\wM_n(s, s') \doteq \sum_{k=1}^{n} K(S_{k-1}, s)\, \ind\{S_{k-1} = s'\},
\quad
\wP_n(s, s') \doteq \sum_{k=1}^{n} K(S_{k-1}, s)\, \ind\{S_k = s'\}.
\end{equation}
Here $\wM_n$ aggregates the softmax weights according to the predecessor 
state $S_{k-1}$, and $\wP_n$ aggregates the same weights according to 
the successor state $S_k$. Then \eqref{eq:update} can be written compactly as
\begin{equation}
\label{eq:operator-compact}
\textstyle
\widehat{\fT}_{\wK}(v) = (I - \wM_n + \gamma \wP_n) v + \wM_n r,
\end{equation}
and stacking layers corresponds to repeated application of $\widehat{\fT}_{\wK}$:
\begin{equation}
\label{eq:recur}
\textstyle
v_{l+1} = \widehat{\fT}_{\wK}(v_l).
\end{equation}

The empirical operator $\widehat{\fT}_{\wK}$ contracts whenever $\wM_n$ has 
sufficiently large diagonal entries. To formalize this, we introduce the 
population counterpart of $\wM_n$, obtained by replacing the empirical 
state frequencies in $K$ with the stationary distribution $\mu_\pi$:
\begin{equation}
\label{eq:Mpi}
\textstyle
M_\pi(s, s') \doteq \frac{\mu_\pi(s')\, \exp(\langle x(s), x(s') \rangle)}
{\sum_{u \in \fS} \mu_\pi(u)\, \exp(\langle x(s), x(u) \rangle)},
\end{equation}
where $\mu_\pi$ is the stationary distribution of $P_\pi$. We impose the 
following contraction assumption.
\begin{assumption}
\label{ass:diagonal-margin}
There exists a constant $C_{\ref{ass:diagonal-margin}} \in 
\bigl(0, \frac{1-\gamma}{2}\bigr)$ such that
\begin{equation}
\textstyle
\min_{s \in \fS} M_\pi(s, s) \geq \frac{1+\gamma}{2} + C_{\ref{ass:diagonal-margin}}.
\end{equation}
\end{assumption}
Assumption~\ref{ass:diagonal-margin} is a kernel diagonality condition on $M_\pi$, requiring  $M_\pi(s,s)$ to dominate the off-diagonal mass at every state $s$. 
It is expected to hold when the score function $g(s,s') = \langle x(s), x(s')\rangle$ separates self-pairs from cross-pairs, i.e., $g(s,s) - g(s,s')$ is sufficiently positive for $s' \neq s$, which sharpens the softmax in \eqref{eq:softmax-weight-sync} toward the diagonal. 
We will see in Section~\ref{sec:experiment} that this self-concentration emerges empirically over pretraining.

When $n$ is large enough, $\wM_n$ inherits the diagonal margin of $M_\pi$, which  together with a high-probability bound on the bias $\|\widehat{\fT}_{\wK}(v_\pi) - v_\pi\|_\infty$ yields our main convergence result.
\begin{theorem}
\label{thm:finite-n}
Under Assumption~\ref{ass:diagonal-margin}, for any $\delta \in (0,1)$, 
there exist constants $C_{\text{Thm}\ref{thm:finite-n},1}, 
C_{\text{Thm}\ref{thm:finite-n},2}, C_{\text{Thm}\ref{thm:finite-n},3} > 0$ 
such that if $n \geq C_{\text{Thm}\ref{thm:finite-n},1} \log(1/\delta)$, 
then with probability at least $1 - \delta$, for all $L \geq 0$,
\begin{equation}
\label{eq:conv}
\textstyle
\|v_L - v_\pi\|_\infty 
\leq 
(1 - C_{\ref{ass:diagonal-margin}})^L \|v_0 - v_\pi\|_\infty 
+ C_{\text{Thm}\ref{thm:finite-n},2}
\sqrt{\frac{\log(C_{\text{Thm}\ref{thm:finite-n},3}/\delta)}{n}}.
\end{equation}
\end{theorem}
The proof is in Section~\ref{proof:finite-n}, which combines an 
$\ell_\infty$-contraction of $\widehat{\fT}_{\wK}$ around $v_\pi$ 
(Lemma~\ref{lem:contraction}) with a high-probability bound on the bias 
$\|\widehat{\fT}_{\wK}(v_\pi) - v_\pi\|_\infty$ (Lemma~\ref{lem:bias}), 
both deferred to Appendix~\ref{proof:contraction} and~\ref{proof:bias} 
respectively. The first term of \eqref{eq:conv} decays geometrically with 
the number of layers $L$, while the second term is a statistical floor 
that vanishes as $n \to \infty$. That is, the softmax ICTD recursion 
converges to $v_\pi$ in the double limit where both depth and context 
length grow.

\section{Emergence of Softmax ICTD under Reinforcement Pretraining}
\label{sec:emergence}

We now answer \ref{(Q3)} by showing that, as the depth $L\to\infty$ and the context length $n \to \infty$, the parameters $\{(V_l,A_l)\}_{l=0}^{L-1}$ constructed in \eqref{eq:para-now} are one of the global minimizers of the reinforcement pretraining loss.

To avoid specializing to a single task, we pretrain on a distribution of MRPs drawn from a joint distribution $\Delta$ over transition probabilities and reward functions, i.e., $(p,r)\sim\Delta$. 
For each sampled MRP, we draw a trajectory $\tau_{n+1} = (S_0, R_1, S_1, \ldots, R_{n+1}, S_{n+1})$ from the stationary Markov chain induced by $\pi$ (i.e., $S_0 \sim \mu_\pi$). 
We construct the prompt $Z_0$ from the first $n$ transitions $(S_0, R_1, \ldots, R_n, S_n)$ as in~\eqref{eq:z_initial}, and denote the scalar output of the $L$-layer Transformer by $v_L(S_n;\theta)$. 
Similarly, $v_L(S_{n+1};\theta)$ denotes the output on the shifted prompt built from $(S_1, R_2, \ldots, R_{n+1}, S_{n+1})$ with query $S_{n+1}$. The TD error at the query state is
\begin{equation}
\label{eq:emerdelta}
\textstyle
\delta_L(\theta) \doteq R_{n+1} + \gamma v_L(S_{n+1};\theta) - v_L(S_n;\theta).
\end{equation}
The norm of the expected update (NEU) is a fundamental objective in TD methods \citep{sutton2009convergent, sutton2009fast}. Its zeros characterize TD fixed points, making it a natural pretraining loss for studying whether the constructed parameters $\theta_*$ attain global optimality. Following \citet{wang2025emergence}, we adopt its extension to arbitrary function approximation, defined as the $\ell_1$-norm of the expected TD update direction:
\begin{align}
\label{eq:NEU}
\textstyle
J_{L,n}(\theta) \doteq 
\bigl\| \E_{(p,r),\, \tau_{n+1}}
[\delta_L(\theta) \nabla_\theta v_L(S_n;\theta)] \bigr\|_1.
\end{align}
Following common practice in the in-context learning theory literature \citep{vonoswald2023transformers, ahn2023transformers,mahankali2024one,zhang2024trained,cheng2024fgd,gatmiry2024looped,huang2025transformers,wang2025emergence}, we analyze NEU on a looped Transformer and restrict the parameters to a structured sparse space. 
Concretely, we tie the parameters across layers by sharing the same pair $(V,A)$, i.e., $(V_l,A_l)=(V,A)$ for all $l\in\{0,\dots,L-1\}$, and search over the following parameter space:
\begin{equation}\label{eq:parameterization_set}
\textstyle
\Theta^{\text{Looped}} \doteq \bigg\{ \theta = (V, A)^L  \bigg| 
 V = \begin{bmatrix} 0_{(d+2) \times d} & 0_{(d+2) \times 3} \\ 
 0_{1 \times d} & u \end{bmatrix}, 
 A = \begin{bmatrix} B & 0_{d \times 3} \\ 
 0_{3 \times d} & 0_{3 \times 3} \end{bmatrix}
\bigg\},
\end{equation}
where $u \in \mathbb{R}^{1 \times 3}$ and $B \in \mathbb{R}^{d \times d}$. 
Note that the parameter choice constructed in \eqref{eq:para-now} lies in $\Theta^{\text{Looped}}$; for convenience, we denote it by $\theta_*$.
We now state the main emergence result.
\begin{theorem}
\label{thm:emergence}
Under Assumption~\ref{ass:diagonal-margin} holding uniformly over $\Delta$, 
there exist a constant $n_0>0$ and constants 
$C_{\text{Thm}\tref{thm:emergence},2}, C_{\text{Thm}\tref{thm:emergence},3}>0$ 
such that for every $L\geq 1$ and $n\geq n_0$,
\begin{equation}
\label{eq:emer}
\textstyle
J_{L,n}(\theta_*) \leq
C_{\text{Thm}\tref{thm:emergence},2}
\qty[
(1-C_{\ref{ass:diagonal-margin}})^L
+
\sqrt{\frac{\log n}{n}}
]
+
C_{\text{Thm}\tref{thm:emergence},3}
\frac{L(2+\gamma)^{2L}}{n}.
\end{equation}
In particular,
\[
\textstyle
\lim_{L\to\infty}\limsup_{n\to\infty}J_{L,n}(\theta_*)=0.
\]
\end{theorem}
The proof is in Section~\ref{proof:emergence}. The argument relies on a uniform bound on $\nabla_\theta v_L(S_n;\theta_*)$ across all depths (Lemma~\ref{lemma:grad_bound} in Appendix~\ref{proof:grad_bound}), which in turn uses Assumption~\ref{ass:diagonal-margin} together with the layer-invariance of $\wK(\theta)$ for $\theta \in \Theta^{\text{Looped}}$. 
Since $J_{L,n}(\theta)\ge 0$ for all $\theta$, Theorem~\ref{thm:emergence} 
implies that $\theta_*$ is a global minimizer of the NEU loss in the  iterated regime where the context length grows first and the depth grows afterward.

\section{Related Works}

The theoretical study of ICRL has largely relied on the linear-attention abstraction, which replaces the softmax with the identity map to make the forward pass tractable \citep{wang2025ictd, wang2025emergence}. 
This simplification is also pervasive in the broader in-context learning (ICL) literature, where it has been used to show that Transformers can implement gradient-based algorithms in their forward pass \citep{ahn2023transformers, zhang2024trained, gatmiry2024looped}. 
A smaller line of work keeps a nonlinearity explicit, either through general activation operators $V Z\, h(Z^\top A Z)$ \citep{cheng2024fgd} or by relating softmax attention to kernel-based learning rules \citep{ren2024representation, dragutinovic2025softmaxgeqlinear}. 
Other results under standard softmax attention show that fixed-weight softmax attention can simulate multi-step gradient descent on deep networks \citep{wu2025incontext}, emulate prompt-programmable algorithms in a two-layer module \citep{hu2026incontext}, or serve as a sequence-to-sequence universal approximator with in-context gradient-descent approximation \citep{hu2026universal}.
While these results inform supervised in-context computation, they do not address reinforcement-pretrained policy evaluation, and the standard softmax used in practical Transformers remains largely unanalyzed in the ICRL setting. 
Our work closes this gap, complementing \citet{wang2025emergence} by removing the linear-attention simplification.

A concurrent submission \citep{xie2026cot} similarly builds on \citet{wang2025ictd, wang2025emergence}, but takes a fundamentally different route: it analyzes linear attention with autoregressive chain-of-thought generation and linear function approximation, where each generated token implements one batch TD update on the context, 
whereas we analyze softmax attention with depth-based computation for finite-state policy evaluation. 
As a result, the two works draw on largely disjoint technical tools. 
For convergence, our analysis is driven by a kernel diagonality condition (Assumption~\ref{ass:diagonal-margin}) and is then lifted to finite contexts via concentration, whereas \citet{xie2026cot} provide a finite-sample analysis on a single Markovian trajectory. 
For emergence, we control how the value prediction depends on the parameters, whereas \citet{xie2026cot} control how the algorithmic iterates depend on the parameters. 
The two works share the Boyan's chain task setup of \citet{wang2025ictd}.

ICRL has been studied under both supervised and reinforcement pretraining \citep{moeini2025survey}. 
Supervised pretraining trains a sequence model to imitate trajectories from existing RL algorithms \citep{laskin2023incontext, liu2023emergent, shi2023cross, kirsch2023towards,  zisman2024emergence,huang2024cot, huang2024decision, dai2024incontext, zisman2025ngram, polubarov2025vintix}, 
with theoretical accounts of the resulting in-context behavior 
\citep{lin2024transformers, lee2023supervised}. 
Reinforcement pretraining instead trains the model with an RL objective and observes that test-time adaptation still emerges \citep{duan2016rl, wang2016learning, mishra2018simple, ritter2018been, stadie2018some, zintgraf2020varibad, melo2022transformer}, with empirical demonstrations across diverse tasks \citep{team2023human, lu2023structured, grigsby2024amago, grigsby2024amago2, elawady2024relic, xu2024meta, cook2024artificial, liu2025locoformer}. 
Theoretical results under reinforcement pretraining remain limited: \citet{park2025llm} analyze a regret-minimization objective, and \citet{wang2025emergence} establish convergence and emergence under linear attention. Our paper falls in this regime, with softmax attention.

\section{Experiments}
\label{sec:experiment}
\paragraph{Setup.}
Following \citet{wang2025ictd}, we study multi-task policy evaluation on randomized instances of Boyan's chain~\citep{boyan1999least}.
At the beginning of each epoch, we sample a Boyan-chain MRP together with a feature map (Algorithm~\tref{alg:representable}).
We construct the prompt $Z_0$ as in~\eqref{eq:z_initial} and apply the Transformer block in Section~\tref{sec:tsm-reparam},
using a column-normalized softmax attention kernel as in~\eqref{eq:repattn}.
As in \citet{wang2025ictd}, the block unrolls autoregressively for $L=3$ layers with shared parameters $(V_l,A_l)=(V_0,A_0)$ across layers.
All entries of $(V_0,A_0)$ are trainable except for a minimal sparse parameterization that prevents overwriting features or rewards; see Section~\ref{app:minimal-mask}.
Hyperparameter and task-generation details are deferred to Section~\tref{app:boyan-setup}.

\begin{figure}[t]
  \centering
  \begin{minipage}[t]{0.43\linewidth}
    \vspace{0pt}
    \centering
    \includegraphics[height=4.2cm,trim=0 2 4 2,clip]{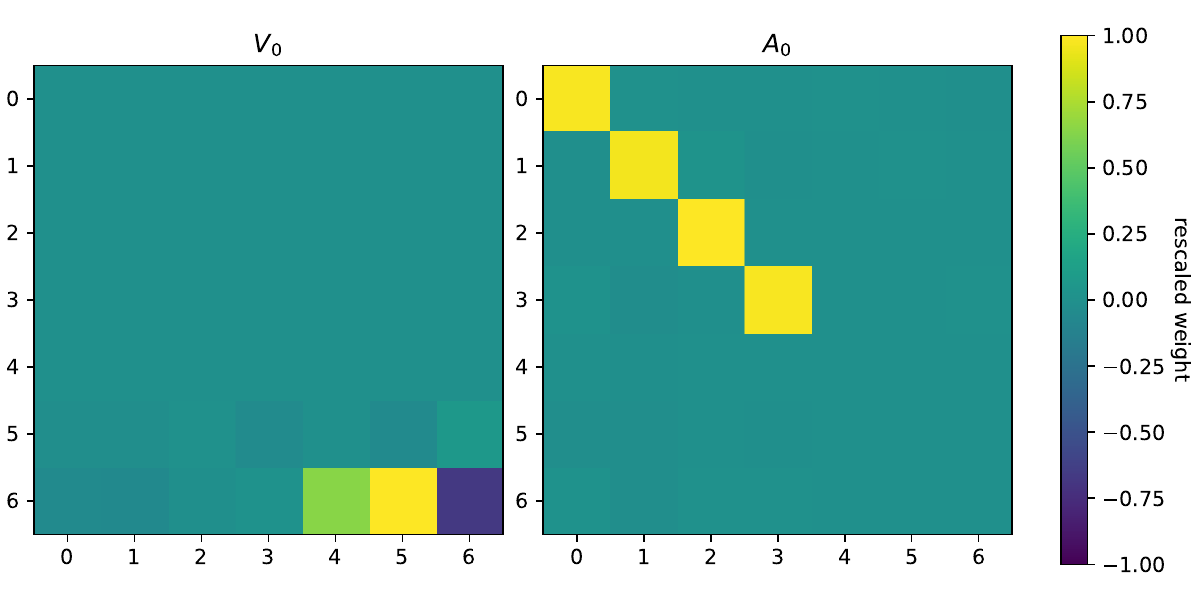}

    \vspace{2pt}
    \small (a) Learned $V_0$ and $A_0$.
  \end{minipage}
  \hfill
  \begin{minipage}[t]{0.53\linewidth}
    \vspace{0pt}
    \centering
    \includegraphics[height=4.2cm,trim=4 2 4 2,clip]{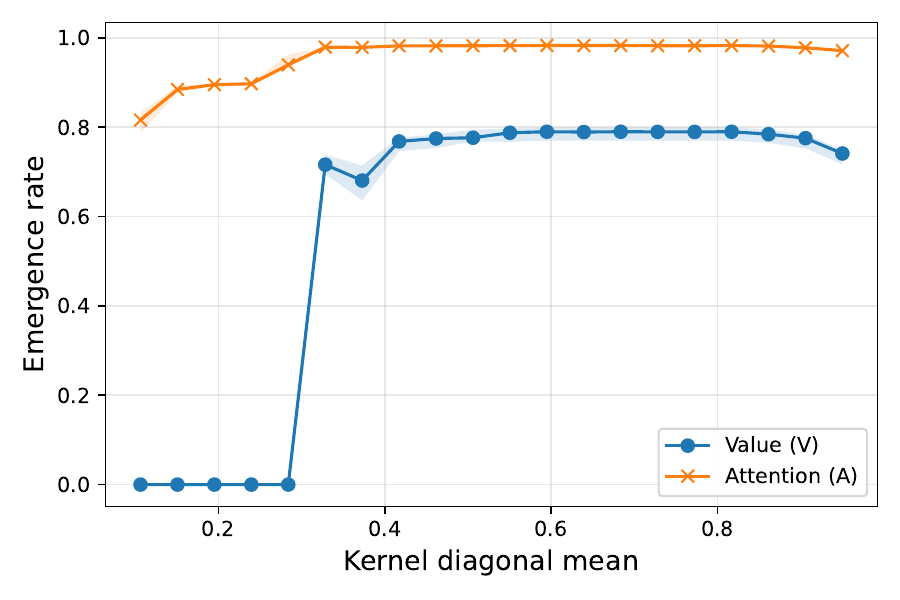}

    \vspace{2pt}
    \small (b) $V_{\mathrm{em}}$ and $A_{\mathrm{em}}$ versus kernel diagonal mean $d_t$.
  \end{minipage}
  \caption{
  Emergence of the learned TD block.
  (a) Learned $V_0$ and $A_0$ at the checkpoint maximizing
  $\min(V_{\mathrm{em}}, A_{\mathrm{em}})$, averaged over $5$ seeds.
  Each matrix is rescaled by its maximum absolute entry for visualization.
  (b) Evolution of $V_{\mathrm{em}}$ and $A_{\mathrm{em}}$ with the kernel diagonal mean $d_t$.
  We record $(d_t, V_{\mathrm{em}}, A_{\mathrm{em}})$ across $3000$ sampled MRPs and aggregate over $5$ seeds;
  shaded regions indicate variability across seeds.
  }
  \label{fig:emergence-main}
\end{figure}

\paragraph{TD emergence.}
We quantify emergence using two scores $V_{\mathrm{em}},A_{\mathrm{em}}\in[0,1]$ that measure the similarity between the learned first-layer parameters $(V_0,A_0)$ and the ideal TD block in~\eqref{eq:para-now}; see Section~\ref{sec:emergence-metrics} for definitions.
Figure~\ref{fig:emergence-main}(a) visualizes $(V_0,A_0)$ at the checkpoint maximizing \(\min(V_{\rm em},A_{\rm em})\) over $3000$ sampled MRPs.
The bottom-right entries of $V_0$ exhibit the $(+,+,-)$ sign pattern on $(r,\gamma v',v)$, closely matching~\eqref{eq:para-now}.
Meanwhile, the upper left $d\times d$ blocks of $A_0$ becomes sharply diagonal, providing a second signature of emergence.

% Figure~\ref{fig:emer} shows the learned first-layer parameters at the checkpoint with maximal $V/A$ emergence, averaged over $10$ seeds.
% In $V_0$ the only visible structure lies in the last row, and the corresponding $A_0$ exhibits a diagonally concentrated state block, both closely matching~\eqref{eq:para-now}.
% Together, these heatmaps indicate that a single ICKTD layer can spontaneously learn a kernel-weighted TD update along the chain without any hard-coded structure.

% To quantify this behavior we define two scalar emergence scores
% $V_{\mathrm{em}}$ and $A_{\mathrm{em}}$ in $[0,1]$, which equal $1$ for an ideal TD block in \eqref{eq:para-now}. 
% The V-emergence score $V_{\mathrm{em}}$ is computed from the rescaled coefficients on $(r,\gamma v',v)$ in the last row of $V_0$. Specifically, we require the sign pattern $(+,+,-)$, 
% and conditional on this we measure how close the normalized triple of coefficients is to the target TD direction $(1,1,-1)$ in both magnitude ratios and angle. 
% The A-emergence score $A_{\mathrm{em}}$ is computed from the $(A_0)_{1:d, 1:d}$, which is also the feature-feature
% submatrix. 

% We normalize each column and measure how much of the absolute
% mass lies on the diagonal, with a small amplitude correction to
% downweight degenerate near-zero matrices.
% Thus, $V_{\mathrm{em}}$ is large only when the value update behaves like a TD triple, and $A_{\mathrm{em}}$ is large only when the attention block is sharply diagonal and
% non-degenerate.
% The exact formulas are given in Section~\ref{sec:emergence-metrics}.

\paragraph{Kernel diagonality.}
Next, we relate emergence to Assumption~\ref{ass:diagonal-margin}.
At training step $t$, let $L^{(t)} \doteq Z_0^\top A_0^{(t)} Z_0 \in \R[n\times n]$.
Define $\wK^{(t)}\in\R[n\times n]$ as the attention matrix 
in~\eqref{eq:kermatrix} with $A = A_0^{(t)}$.
% Define $\wK^{(t)}\in\R[n\times n]$ as the column-normalized softmax kernel obtained by setting $G=L^{(t)}$ in~\eqref{eq:softmax}.
We summarize its kernel diagonality by the kernel diagonal mean
\begin{equation}
\label{eq:kernel-diag-mean}
\textstyle
    d_t \doteq \frac{1}{n}\tr(\wK^{(t)}).
\end{equation}
% where $d_t = 1/n$ corresponds to a flat kernel and $d_t = 1$ to an identity kernel.
Figure~\ref{fig:emergence-main}(b) plots $V_{\mathrm{em}}$ and $A_{\mathrm{em}}$ against $d_t$.
We find that $A_{\mathrm{em}}$ remains high across a wide range of $d_t$, whereas $V_{\mathrm{em}}$ starts to rise only after $d_t$ reaches an intermediate level.
This suggests that learning the TD update direction in the bottom-right entries of $V$ requires a sufficient kernel diagonality.

\section{Conclusion}
\label{sec:conclusion}
We provide the first theoretical analysis of in-context policy evaluation beyond the linear-attention simplification, under softmax attention. 
We show by construction that a Transformer layer can be configured to implement a weighted softmax TD update step in its forward pass. 
We prove an inference-time convergence guarantee as the number of layers increases under a suitable contraction condition, and show that the desired parameters are one of the global minimizers of a pretraining loss, explaining their emergence with increasingly diagonal attention in our experiments. 
Future work includes relaxing the contraction requirement and extending the theory and experiments to sampled trajectories and control settings.

\section*{Acknowledgments}
This work is supported in part by the US National Science Foundation under the awards III-2128019, SLES-2331904, and CAREER-2442098, the Commonwealth Cyber Initiative's Central Virginia Node under the award VV-1Q26-001, and a Cisco Faculty Research Award.

{\small
\bibliography{bibliography}
}

%%%%%%%%%%%%%%%%%%%%%%%%%%%%%%%%%%%%%%%%%%%%%%%%%%%%%%%%%%%%%%%%%%%%%%%%%%%%%%%
%%%%%%%%%%%%%%%%%%%%%%%%%%%%%%%%%%%%%%%%%%%%%%%%%%%%%%%%%%%%%%%%%%%%%%%%%%%%%%%
% APPENDIX
%%%%%%%%%%%%%%%%%%%%%%%%%%%%%%%%%%%%%%%%%%%%%%%%%%%%%%%%%%%%%%%%%%%%%%%%%%%%%%%
%%%%%%%%%%%%%%%%%%%%%%%%%%%%%%%%%%%%%%%%%%%%%%%%%%%%%%%%%%%%%%%%%%%%%%%%%%%%%%%
\newpage
\appendix
\onecolumn

\section{Auxiliary Lemmas}
\label{sec:aux_lemma}
\begin{lemma}[Theorem~12 of \citet{fan2021hoeffding}]
\label{lem:markov-hoeffding}
Let $\{S_k\}$ be a Markov chain on the finite state space $\fS$ 
with transition matrix $P_\pi$ and initial distribution $S_0 \sim p_0$.
Assume the Markov chain induced by $P_\pi$ is ergodic 
with stationary distribution $\mu_\pi$.
Let $P_\pi^* = \diag(\mu_\pi)^{-1} P_\pi^\top \diag(\mu_\pi)$ 
denote the time reversal of $P_\pi$, and define the 
additive reversiblization $R_\pi \doteq (P_\pi + P_\pi^*)/2$.
Let $\lambda_r$ denote the second-largest eigenvalue of $R_\pi$,
so that $1 - \lambda_r$ is the right spectral gap.
Then for any function $f \colon \fS \to [a, b]$ and any $\epsilon > 0$,
\begin{equation}
\label{eq:concentration}
  \Pr\!\bigg(\,\bigg|\frac{1}{n}\sum_{k=0}^{n-1} f(S_k) 
  - \mu_\pi(f)\bigg| > \epsilon\,\bigg)
  \le
  \frac{2}{\mu_{\min}}\,
  \exp\!\bigg(\!
  -\frac{1-\max\{\lambda_r, 0\}}{1+\max\{\lambda_r, 0\}}
  \cdot\frac{2n\epsilon^2}{(b-a)^2}\bigg),
\end{equation}
where $\mu_{\min} \doteq \min_{s \in \fS} \mu_\pi(s) > 0$.
For any finite irreducible aperiodic chain, 
$R_\pi$ is a finite reversible irreducible chain, 
so $\lambda_r < 1$ and the bound is nontrivial.
\end{lemma}

\begin{remark}
This is a special case of Theorem~12 in \citet{fan2021hoeffding}
with burn-in length $n_0 = 0$ and integrability exponent $p = \infty$, 
so that the prefactor becomes
$C(\nu, 0, \infty) = \lVert p_0 / \mu_\pi \rVert_\infty 
\le 1/\mu_{\min}$.
\end{remark}

\begin{lemma}[Proposition~A.2 of \citet{ortega2000iterative}]
\label{lem:mv-ineq}
Let $f: \R[n] \to \R[m]$ be continuously differentiable.
Then for any $w, w' \in \R[n]$,
\begin{equation}
\|f(w) - f(w')\|_\infty
\;\le\;
\sup_{t \in [0,1]} \|Df(w + t(w'-w))\|_\infty
\;\|w - w'\|_\infty,
\end{equation}
where $Df(z) \in \R[m \times n]$ denotes the Jacobian of $f$ at $z$.
\end{lemma}

\section{Proof of Theorem~\ref{thm:onestep}}
\label{proof:onestep}

\begin{proof}
We prove by induction that, for every layer $l\ge 0$, the prompt admits the form
\begin{equation}
\label{eq:z_general}
Z_l \doteq \begin{bmatrix}
    x_0 & x_1 & \cdots & x_{n-1} & x_n \\
    R_1 & R_2 & \cdots & R_n & 0 \\
    \gamma v_l(S_1) & \gamma v_l(S_2) & \cdots & \gamma v_l(S_n) & 0 \\
    v_l(S_0) & v_l(S_1) & \cdots & v_l(S_{n-1}) & v_l(S_n)
\end{bmatrix}.
\end{equation}
This holds at $l=0$ by the definition of $Z_0$ in \eqref{eq:z_initial}.

Now fix any $l\ge 0$ and assume that $Z_l$ satisfies \eqref{eq:z_general}.
We derive the layer update induced by $(V_l,A_l)$ in \eqref{eq:para-now}.
Since $A_l=\diag(I_d,0)$, the score matrix depends only on the first $d$ rows of $Z_l$.
Under \eqref{eq:z_general}, the first $d$ entries of the $(j+1)$-th column of $Z_l$ equal $x_j$ for every $j=0,\dots,n$.
Hence
\begin{equation}
\label{eq:score-gram}
(Z_l^\top A_l Z_l)[i,j] = \langle x_{i-1}, x_{j-1} \rangle,
\quad i,j=1,\dots,n+1.
\end{equation}
Therefore the attention matrix $\widetilde K_l$ is layer-invariant (denoted as $\widetilde K$).
By softmax weights $K$ in \eqref{eq:softmax-weight-sync}, for every $j=0,\dots,n$ we have
\begin{align}
\wK[k,j+1] = 
    \begin{cases}
        K(S_{k-1},S_j), & k=1,\dots,n, \\
        0, & k=n+1
    \end{cases}
\end{align}
% for every $j=0,\dots,n$ and $k=1,\dots,n$,
% \[
% \widetilde K[k,j+1] = K(S_{k-1},S_j),
% \]
% and
% \[
% \widetilde K[n+1,j]=0,\qquad j=1,\dots,n+1.
% \]

Next, consider the current-value head.
For each context column $k=1,\dots,n$,
\begin{align}
\label{eq:current_value_head_output}
(V_l Z_l)[d+3,k]
&=
\begin{bmatrix}
0_{1\times d} & 1 & 1 & -1
\end{bmatrix}
\begin{bmatrix}
x_{k-1}\\
R_k\\
\gamma v_l(S_k)\\
v_l(S_{k-1})
\end{bmatrix} \\
&= R_k + \gamma v_l(S_k) - v_l(S_{k-1}) \\
&= \delta_k^{(l)}\quad \text{(By \eqref{eq:tderror})}.
\end{align}
For the last column,
\[
(V_l Z_l)[d+3,n+1]
=
\begin{bmatrix}
0_{1\times d} & 1 & 1 & -1
\end{bmatrix}
\begin{bmatrix}
x_n\\
0\\
0\\
v_l(S_n)
\end{bmatrix}
= -v_l(S_n).
\]
Define
\[
\widehat{\delta}^{(l)}
\doteq
[\delta_1^{(l)},\dots,\delta_n^{(l)},-v_l(S_n)],
\qquad
\boldsymbol{\delta}^{(l)}
\doteq
[\delta_1^{(l)},\dots,\delta_n^{(l)},0].
\]
Since the last row of $\widetilde K$ is zero, we have
\[
\widehat{\delta}^{(l)}\widetilde K
=
\boldsymbol{\delta}^{(l)}\widetilde K.
\]
Therefore,
\begin{equation}
\label{eq:proof-delta-current}
\Delta Z_l[d+3,:]
=
\big(V_l Z_l\,\softmax(Z_l^\top A_l Z_l)\big)[d+3,:]
=
\boldsymbol{\delta}^{(l)}\widetilde K.
\end{equation}

For the target-value head, \eqref{eq:tfupdate-prev} gives
\begin{equation}
\label{eq:proof-delta-target}
\Delta Z_l[d+2,:]
=
\gamma\,
\big(V_l Z_l\,\softmax(Z_l^\top A_l Z_l)\Pi\big)[d+3,:]
=
\gamma\,\boldsymbol{\delta}^{(l)}\widetilde K \Pi.
\end{equation}

We now read out the last column.
Using \eqref{eq:outputv} and \eqref{eq:proof-delta-current},
\begin{align*}
v_{l+1}(S_n)
&= Z_{l+1}[d+3,n+1] \\
&= Z_l[d+3,n+1] + \Delta Z_l[d+3,n+1] \\
&= v_l(S_n) + \sum_{k=1}^n \delta_k^{(l)} \widetilde K[k,n+1] \\
&= v_l(S_n) + \sum_{k=1}^n \delta_k^{(l)} K(S_{k-1},S_n),
\end{align*}
which is exactly \eqref{eq:update}.

It remains to verify that $Z_{l+1}$ preserves the form in \eqref{eq:z_general}.
By construction, only rows $d+2$ and $d+3$ are modified, so the first $d+1$ rows remain unchanged.
Moreover, for every $j=0,\dots,n$, using \eqref{eq:proof-delta-current},
\[
Z_{l+1}[d+3,j+1]
=
v_l(S_j) + \sum_{k=1}^n \delta_k^{(l)} K(S_{k-1},S_j)
=
v_{l+1}(S_j).
\]
Thus the last row stores the synchronous weighted softmax TD updates for all trajectory states.

Finally, right multiplication by $\Pi$ shifts columns to the left.
Hence for every $j=1,\dots,n$, using \eqref{eq:proof-delta-target},
\begin{align*}
Z_{l+1}[d+2,j]
&=
Z_l[d+2,j] + \gamma(\boldsymbol{\delta}^{(l)}\widetilde K)_{j+1} \\
&=
\gamma v_l(S_j) + \gamma\big(v_{l+1}(S_j)-v_l(S_j)\big) \\
&=
\gamma v_{l+1}(S_j).
\end{align*}
Also, $Z_{l+1}[d+2,n+1]=0$ because the last column of $\Pi$ is zero.
Therefore $Z_{l+1}$ again satisfies \eqref{eq:z_general}, completing the induction.
\end{proof}

\section{Proofs in Section~\tref{sec:tsm-reparam}}
\label{sec:proofsec3}
\subsection{Proof of Lemma~\tref{lem:equiv}}
\label{proof:equiv}
\begin{proof}
From \eqref{eq:z_general}, $Z_l$ satisfies that $Z_l[d+2, :] = \gamma Z_l[d+3, :]\Pi$. Using this invariant, by \eqref{eq:zhalf}, we have
\[
Z_{l+\frac12}
=
Z_l + \Delta Z_l^{(d+3)},
\]
where $\Delta Z_l^{(d+3)}$ is defined in \eqref{eq:tfupdate}.
By \eqref{eq:tsm-forward}, the TSM step leaves the first $(d+1)$ rows unchanged and,
for every column $k$,
\[
(Z_{l+1})_{d+3,k}=(Z_{l+\frac12})_{d+3,k},
\qquad
(Z_{l+1})_{d+2,:}=\gamma (Z_{l+\frac12})_{d+3,:}\Pi,
\]
with $\Pi$ defined in \eqref{eq:pi}.
Therefore $Z_{l+1}=Z_{l+\frac12}+\Delta Z_l^{(d+2)}$, where $\Delta Z_l^{(d+2)}$
coincides with \eqref{eq:tfupdate-prev}.
Thus,
\[
Z_{l+1}
=
Z_l + \Delta Z_l^{(d+3)} + \Delta Z_l^{(d+2)},
\]
which is exactly \eqref{eq:dualhead-out}. Hence the dual-head update equals the
single-head+TSM update in \eqref{eq:tsm-forward}.
\end{proof}

% \subsection{Proof of Lemma~\tref{lem:complexity}}
% \label{proof:complexity}
% \begin{proof}
% Let $T\doteq n+1$ and $p\doteq d+3$, so $Z\in\R[p\times T]$. We count scalar
% operations up to constant factors.
% A single head computes
% \[
% \Delta Z = V Z M\, \tilde h(Z^\top A Z),
% \]
% where $M\in\R[T\times T]$ is diagonal and $\tilde h(\cdot)$ maps a score matrix in $\R[T\times T]$ to a dense weight matrix in $\R[T\times T]$.
% Under \eqref{eq:para-now}, $A=\diag(I_d,0)$ projects onto the first $d$ rows of $Z$. 
% Let $X\in\R[d\times T]$ denote this block. Then $Z^\top A Z=X^\top X$,
% which costs $O(dT^2)=O(pT^2)$ to form. Computing $\tilde h(X^\top X)$ costs
% $O(T^2)$ and is dominated by $O(pT^2)$.

% On the value side, forming $VZ$ costs $O(pT)$ since $V$ in \eqref{eq:para-now}
% has $O(1)$ nonzeros per column, and multiplying by the diagonal $M$ costs another
% $O(pT)$. Multiplying the resulting $p\times T$ matrix by the dense $T\times T$
% weight matrix costs $O(pT^2)$. Hence one head costs $O(pT^2)=O(n^2p)$.

% \paragraph{Dual-head.}
% The dual-head block evaluates two such attention passes, hence costs
% $2\cdot O(pT^2)=2\cdot O(n^2p)$.
% \paragraph{Single-head + TSM.}
% The reparameterized block evaluates attention once, costing $O(pT^2)$.
% It then applies the fixed operators in \eqref{eq:tsm-forward}. Right-multiplication
% by $\Pi$ is a one-step column shift and can be implemented in $O(pT)=O(np)$ time,
% and the row operators $U$ and $W$ touch $O(pT)$ entries as well. Thus the TSM
% overhead is $O(pT)=O(np)$, and the total cost is $O(pT^2)+O(pT)=O(n^2p)+O(np)$.
% \end{proof}

\section{Proofs of Section~\tref{sec:convergence}}
\label{sec:proofsec4}
\label{proof:convergence}
\subsection{Proof of Lemma~\tref{lem:concentration-Mn}}
\begin{lemma}
\label{lem:concentration-Mn}
Recall $\wM$ in \eqref{eq:Mn} and $M_\pi$ in \eqref{eq:Mpi}.
For any $\delta \in (0,1)$, there exist constants
$C_{\ref{lem:concentration-Mn},1}, C_{\ref{lem:concentration-Mn},2}> 0$ 
such that with probability at least $1 - \delta$,
\begin{equation}
\|\wM_n - M_\pi\|_\infty 
\leq C_{\ref{lem:concentration-Mn},1} \sqrt{\frac{\log(C_{\ref{lem:concentration-Mn},2}/\delta)}{n}}.
\end{equation}
\end{lemma}
\begin{proof}
Define the empirical state frequency
\begin{equation}
\label{eq:statefreq}
    \hat{\mu}_n(s) \doteq \frac{1}{n}\sum_{k=1}^n \ind\{S_{k-1}=s\}.
\end{equation}
Applying Lemma~\ref{lem:markov-hoeffding} with $f = \ind\{\cdot = s\} \in [0,1]$
and a union bound over $s \in \fS$ gives
\begin{equation}
\label{eq:freq-conc}
\Pr\!\Big(\|\hat{\mu}_n - \mu_\pi\|_\infty > \epsilon\Big)
\;\le\;
\frac{2|\fS|}{\mu_{\min}}
\exp\!\bigg(\!-\frac{1-\max\{\lambda_r, 0\}}{1+\max\{\lambda_r, 0\}}\cdot 2n\epsilon^2\bigg).
\end{equation}
Setting the right-hand side equal to $\delta$ and solving for $\epsilon$ yields that
with probability at least $1 - \delta$,
\begin{equation}
\label{eq:freq-rate}
\|\hat{\mu}_n - \mu_\pi\|_\infty
\le
\sqrt{\frac{1+\max\{\lambda_r, 0\}}{2(1-\max\{\lambda_r, 0\})}}
\sqrt{\frac{\log(2|\fS|/(\mu_{\min}\,\delta))}{n}}
\doteq
C_{\ref{lem:concentration-Mn},3}\sqrt{\frac{\log(C_{\ref{lem:concentration-Mn},4}/\delta)}{n}},
\end{equation}
where 
$C_{\ref{lem:concentration-Mn},3} \doteq \sqrt{\frac{1+\max\{\lambda_r,0\}}{2(1-\max\{\lambda_r,0\})}}$
and
$C_{\ref{lem:concentration-Mn},4} \doteq \frac{2|\fS|}{\mu_{\min}}$.

It remains to transfer \eqref{eq:freq-rate} to $\wM_n$.
By expanding \eqref{eq:Mn} with \eqref{eq:statefreq}, we obtain
\begin{equation}
\wM_n(s, s') 
= \frac{\hat{\mu}_n(s')\, \exp(\langle x(s), x(s') \rangle)}
{\sum_{u \in \fS} \hat{\mu}_n(u)\, \exp(\langle x(s), x(u) \rangle)}.
\end{equation}
Let $C_x \doteq \max_{s \in \fS} \|x(s)\|_2$.
A direct calculation then gives
\begin{align}
&\left|\wM_n(s,s') - M_\pi(s,s')\right| \notag\\
={}&
\exp(\langle x(s), x(s') \rangle)
\left|
\frac{\hat{\mu}_n(s')\sum_u \mu_\pi(u)\,\exp(\langle x(s),x(u)\rangle)
      - \mu_\pi(s')\sum_u \hat{\mu}_n(u)\,\exp(\langle x(s),x(u)\rangle)}
{\sum_u \hat{\mu}_n(u)\,\exp(\langle x(s),x(u)\rangle)
 \cdot \sum_u \mu_\pi(u)\,\exp(\langle x(s),x(u)\rangle)}
\right| \notag\\
\le{}&
C_{\ref{lem:concentration-Mn},5}\,\|\hat{\mu}_n - \mu_\pi\|_\infty,
\end{align}
where $C_{\ref{lem:concentration-Mn},5} \doteq 2\exp(4C_x^2)$.
Summing over $s' \in \fS$ and taking the maximum over $s \in \fS$ gives
\begin{equation}
\|\wM_n - M_\pi\|_\infty
\le
|\fS|\,C_{\ref{lem:concentration-Mn},5}\,
C_{\ref{lem:concentration-Mn},3}\sqrt{\frac{\log(C_{\ref{lem:concentration-Mn},4}/\delta)}{n}}.
\end{equation}
Setting
$C_{\ref{lem:concentration-Mn},1}
\doteq |\fS|\,C_{\ref{lem:concentration-Mn},3}\,C_{\ref{lem:concentration-Mn},5}$
and
$C_{\ref{lem:concentration-Mn},2}
\doteq C_{\ref{lem:concentration-Mn},4}$
completes the proof.
\end{proof}

\subsection{Proof of Lemma~\tref{lem:contraction}}
\label{proof:contraction}
\begin{lemma}
\label{lem:contraction}
Under Assumption~\ref{ass:diagonal-margin}, 
for any $\delta \in (0,1)$, 
there exist a constant $C_{\ref{lem:contraction}}> 0$ such that 
if $n \geq C_{\ref{lem:contraction}} \log(1/\delta)$, 
then with probability at least $1 - \delta/2$,
\begin{equation}
\|\widehat{\fT}_{\wK}(v) - \widehat{\fT}_{\wK}(u)\|_\infty 
\leq (1- C_{\ref{ass:diagonal-margin}}) \|v - u\|_\infty 
\quad \forall\, u, v \in \R[|\fS|].
\end{equation}
\end{lemma}
\begin{proof}
By \eqref{eq:operator-compact}, for any $u, v \in \R[|\fS|]$,
\begin{equation}
\wT_{\wK}(v) - \wT_{\wK}(u) 
= (I - \wM_n + \gamma \wP_n)(v - u).
\end{equation}
% It suffices to show $\|I - \wM_n + \gamma \wP_n\|_\infty \leq 1 - C_{\ref{lem:contraction}}$.
Since $\wM_n(s, s') \geq 0$ with $\sum_{s'} \wM_n(s, s') = 1$,
and $\wP_n(s, s') \geq 0$ with $\sum_{s'} \wP_n(s, s') = 1$,
the $\ell_1$-norm of row $s$ satisfies
\begin{align}
\sum_{s'} \big|(I - \wM_n + \gamma \wP_n)(s, s')\big|
&\leq \big(1 - \wM_n(s, s)\big) 
  + \sum_{s' \neq s} \wM_n(s, s') 
  + \gamma \sum_{s'} \wP_n(s, s') \notag\\
&= 2\big(1 - \wM_n(s, s)\big) + \gamma.
\end{align}
Choosing
$C_{\ref{lem:contraction}} 
\doteq (2C_{\ref{lem:concentration-Mn}}/C_{\ref{ass:diagonal-margin}})^2$, 
then by Lemma~\ref{lem:concentration-Mn}, 
with probability at least $1 - \delta/2$,
\begin{equation}
\|\wM_n - M_\pi\|_\infty 
\leq C_{\ref{lem:concentration-Mn},1}
\sqrt{\frac{\log(2C_{\ref{lem:concentration-Mn},2}/\delta)}{n}}
\leq \frac{C_{\ref{ass:diagonal-margin}}}{2},
\end{equation}
On this event, Assumption~\ref{ass:diagonal-margin} gives
\begin{equation}
\wM_n(s, s) 
\geq M_\pi(s, s) - \frac{C_{\ref{ass:diagonal-margin}}}{2}
\geq \frac{1 + \gamma}{2} + C_{\ref{ass:diagonal-margin}} - \frac{C_{\ref{ass:diagonal-margin}}}{2}
= \frac{1 + \gamma}{2} + \frac{C_{\ref{ass:diagonal-margin}}}{2}.
\end{equation}
Therefore,
\begin{equation}
\|I - \wM_n + \gamma \wP_n\|_\infty 
= \max_{s \in \fS} \sum_{s'} \big|(I - \wM_n + \gamma \wP_n)(s, s')\big|
\leq 2\bigg(1 - \frac{1+\gamma}{2} - \frac{C_{\ref{ass:diagonal-margin}}}{2}\bigg) + \gamma
= 1 - C_{\ref{ass:diagonal-margin}},
\end{equation}
which completes the proof.
\end{proof}

\subsection{Proof of Lemma~\tref{lem:bias}}
\label{proof:bias}
\begin{lemma}
\label{lem:bias}
For any $\delta \in (0,1)$, 
there exists a constant $C_{\ref{lem:bias},1}, C_{\ref{lem:bias},2} > 0$ such that with probability at least $1 - \delta/2$,
\begin{equation}
T_2
\leq C_{\ref{lem:bias},1} \sqrt{\frac{\log(C_{\ref{lem:bias},2}/\delta)}{n}}.
\end{equation}
\end{lemma}
\begin{proof}
For each $s \in \fS$, expanding $\wP_n$ and $\wM_n$ in \eqref{eq:Mn} gives
\begin{align}
\label{eq:T2-grouped}
\abs{[(\wP_n - \wM_n P_\pi)\,v_\pi](s)}
&= \abs{\sum_{k=1}^n K(S_{k-1}, s)
   \bigl[v_\pi(S_k) - (P_\pi v_\pi)(S_{k-1})\bigr]} \notag\\
% &= \frac{
%    \abs{\sum_{s' \in \fS} \exp(\langle x(s), x(s')\rangle)
%    \sum_{k=1}^n \ind\{S_{k-1}=s'\}
%    \bigl[v_\pi(S_k) - (P_\pi v_\pi)(s')\bigr]}
% }{
%    \sum_{m=1}^n \exp(\langle x(s), x(S_{m-1})\rangle)
% } \notag\\
&\leq \frac{
   \sum_{s' \in \fS} \exp(\langle x(s), x(s')\rangle)
   \abs{\sum_{k=1}^n \ind\{S_{k-1}=s'\}
   \bigl[v_\pi(S_k) - (P_\pi v_\pi)(s')\bigr]}
}{
   \sum_{m=1}^n \exp(\langle x(s), x(S_{m-1})\rangle)
} \notag\\
&\leq \frac{C_{\ref{lem:bias},3}}{n}\;
\max_{s' \in \fS}\;
\abs{\sum_{k=1}^n \ind\{S_{k-1}=s'\}
\bigl[v_\pi(S_k) - (P_\pi v_\pi)(s')\bigr]},
\end{align}
where $C_{\ref{lem:bias},3}\doteq|\fS|\exp(2C_x^2)$.
For each $s' \in \fS$, define
$g_{s'}\colon \fS \times \fS \to \R$ by
\begin{equation}
g_{s'}(a, b)
\doteq \ind\{a = s'\}\bigl[v_\pi(b) - (P_\pi v_\pi)(s')\bigr],
\end{equation}
so that $g_{s'} \in [-2\|v_\pi\|_\infty,\; 2\|v_\pi\|_\infty]$
and the inner sum in \eqref{eq:T2-grouped} equals
$\sum_{k=1}^n g_{s'}(S_{k-1}, S_k)$.
The pair process $(S_{k-1}, S_k)_{k \geq 1}$ is a Markov chain
on the support $E \doteq \{(a,b) \in \fS \times \fS : P_\pi(a,b) > 0\}$.
Since $P_\pi$ is ergodic on $\fS$, the pair chain is ergodic on $E$
with stationary distribution
$\widetilde\pi(a,b) = \mu_\pi(a)\,P_\pi(a,b)$, under which
\begin{align}
\widetilde \pi(g_{s'})
&= \sum_{a,b} \mu_\pi(a)\,P_\pi(a,b)\,
   \ind\{a = s'\}\bigl[v_\pi(b) - (P_\pi v_\pi)(s')\bigr] \notag\\
&= \mu_\pi(s')\bigl[(P_\pi v_\pi)(s') - (P_\pi v_\pi)(s')\bigr]
= 0.
\end{align}
Let $\widetilde\lambda_r$ denote the right spectral gap parameter 
of the pair chain as defined in Lemma~\ref{lem:markov-hoeffding} and
$\widetilde\mu_{\min} \doteq \min_{(a,b) \in E}\,\mu_\pi(a)\,P_\pi(a,b) > 0$.
Applying Lemma~\ref{lem:markov-hoeffding}
to the pair chain with function $g_{s'}$
and a union bound over $s' \in \fS$,
with probability at least $1 - \delta/2$,
\begin{equation}
\label{eq:pair-conc}
\max_{s' \in \fS}\;
\abs{\frac{1}{n}\sum_{k=1}^n g_{s'}(S_{k-1}, S_k)}
\le 2\|v_\pi\|_\infty
\sqrt{\frac{2(1+\max\{\widetilde\lambda_r, 0\})}{1-\max\{\widetilde\lambda_r, 0\}}}
\sqrt{\frac{\log(4|\fS|/(\widetilde\mu_{\min}\,\delta))}{n}}.
\end{equation}
Substituting \eqref{eq:pair-conc} into \eqref{eq:T2-grouped}
and taking the max over $s \in \fS$ gives
\begin{align}
T_2
&= \gamma\,\max_{s \in \fS}\;
   \abs{[(\wP_n - \wM_n P_\pi)\,v_\pi](s)}\\
&\leq \frac{\gamma C_{\ref{lem:bias},3}}{n}\;
\max_{s' \in \fS}\;
\abs{\sum_{k=1}^n g_{s'}(S_{k-1}, S_k)}\\
&\le C_{\ref{lem:bias},1}
   \sqrt{\frac{\log(C_{\ref{lem:bias},2}/\delta)}{n}},
\end{align}
where
\begin{equation}
\label{eq:Cbias}
C_{\ref{lem:bias},1}
\doteq 2\gamma \|v_\pi\|_\infty
\sqrt{\frac{2(1+\max\{\widetilde\lambda_r, 0\})}{1-\max\{\widetilde\lambda_r, 0\}}}
C_{\ref{lem:bias},3},
\quad
C_{\ref{lem:bias},2}
\doteq \frac{4|\fS|}{\widetilde\mu_{\min}}.
\end{equation}
This completes the proof.
\end{proof}

\subsection{Proof of Theorem~\tref{thm:finite-n}}
\label{proof:finite-n}
\begin{proof}
We decompose the error at each layer as
\begin{align}
\|v_{l+1} - v_\pi\|_\infty 
&= \|\widehat{\fT}_{\wK}(v_l) - v_\pi\|_\infty \notag\\
&\leq \underbrace{\|\widehat{\fT}_{\wK}(v_l) - \widehat{\fT}_{\wK}(v_\pi)\|_\infty}_{T_1}
+ \underbrace{\|\widehat{\fT}_{\wK}(v_\pi) - v_\pi\|_\infty}_{T_2} \notag\\
&\leq \underbrace{\|(I - \wM_n + \gamma \wP_n)(v_l - v_\pi)\|_\infty}_{T_1}
+ \underbrace{\|\gamma(\wP_n - \wM_n P_\pi)\, v_\pi\|_\infty}_{T_2}, \label{eq:error-decomp}
\end{align}
where the second inequality uses \eqref{eq:operator-compact} for $T_1$ 
and the Bellman equation $r + \gamma P_\pi v_\pi = v_\pi$ for $T_2$.
% By \eqref{eq:error-decomp}, the error at each layer decomposes as
% \begin{align}
% \|v_{l+1} - v_\pi\|_\infty 
% &= \|\wT_{\wK}(v_l) - v_\pi\|_\infty \notag\\
% &\leq \underbrace{\|\wT_{\wK}(v_l) - \wT_{\wK}(v_\pi)\|_\infty}_{T_1}
%      + \underbrace{\|\wT_{\wK}(v_\pi) - v_\pi\|_\infty}_{T_2}.
% \end{align}
% By Lemma~\ref{lem:contraction},
with probability at least $1 - \delta/2$,
if $n \geq C_{\ref{lem:contraction}}\log(1/\delta)$,
then for all $v_l$,
\begin{equation}
\label{eq:T1-bound}
T_1 \leq (1 - C_{\ref{ass:diagonal-margin}})\|v_l - v_\pi\|_\infty.
\end{equation}
By Lemma~\ref{lem:bias},
with probability at least $1 - \delta/2$,
\begin{equation}
\label{eq:T2-bound}
T_2 \leq C_{\ref{lem:bias},1}
\sqrt{\frac{\log(C_{\ref{lem:bias},2}/\delta)}{n}}.
\end{equation}
A union bound ensures that \eqref{eq:T1-bound} and \eqref{eq:T2-bound}
hold simultaneously with probability at least $1 - \delta$.
On this event, iterating over $L$ layers gives
\begin{align}
\|v_L - v_\pi\|_\infty
&\leq (1 - C_{\ref{ass:diagonal-margin}})^L \|v_0 - v_\pi\|_\infty
+ \sum_{l=0}^{L-1} (1 - C_{\ref{ass:diagonal-margin}})^l \cdot C_{\ref{lem:bias},1}
\sqrt{\frac{\log(C_{\ref{lem:bias},2}/\delta)}{n}} \notag\\
&= (1 - C_{\ref{ass:diagonal-margin}})^L \|v_0 - v_\pi\|_\infty
+ \frac{1 - (1 - C_{\ref{ass:diagonal-margin}})^L}{C_{\ref{ass:diagonal-margin}}}
\cdot C_{\ref{lem:bias},1}
\sqrt{\frac{\log(C_{\ref{lem:bias},2}/\delta)}{n}} \notag\\
&\leq (1 - C_{\ref{ass:diagonal-margin}})^L \|v_0 - v_\pi\|_\infty
+ \frac{C_{\ref{lem:bias},1}}{C_{\ref{ass:diagonal-margin}}}
\sqrt{\frac{\log(C_{\ref{lem:bias},2}/\delta)}{n}}.
\end{align}
Choosing
$C_{\text{Thm}\ref{thm:finite-n},1} \doteq C_{\ref{lem:contraction}}$,
$C_{\text{Thm}\ref{thm:finite-n},2} \doteq 
C_{\ref{lem:bias},1}/C_{\ref{ass:diagonal-margin}}$,
and $C_{\text{Thm}\ref{thm:finite-n},3} \doteq C_{\ref{lem:bias},2}$
completes the proof.
\end{proof}

\section{Proofs of Section~\tref{sec:emergence}}
\label{sec:proofsec5}
\subsection{Proof of Lemma~\tref{lem:decay}}
\begin{lemma}
\label{lem:decay}
Under Assumption~\ref{ass:diagonal-margin}, for any $\delta\in(0,1)$,
there exist constants $C_{\tref{lem:decay},1},\;C_{\tref{lem:decay},2},\;C_{\tref{lem:decay},3}>0$
such that if $n \ge C_{\tref{lem:decay},1}\log(1/\delta)$,
then with probability at least $1-\delta$,
for any $(p,r)$ and query state $S_n$,
\begin{equation}
\label{eq:decay_td_cond}
\abs{\E_{S_{n+1}\sim p(\cdot\mid S_n)}\big[\delta_L(\theta_*)\big]}
\le (1+\gamma)\left[
(1-C_{\ref{ass:diagonal-margin}})^{L}\|v_\pi\|_\infty
\;+\;
C_{\tref{lem:decay},2}\sqrt{\frac{\log(C_{\tref{lem:decay},3}/\delta)}{n}}
\right].
\end{equation}
\end{lemma}
\begin{proof}
By \eqref{eq:emerdelta} and the linearity of expectation we have
\begin{align}
\label{eq:decay_relation}
&\abs{\mathbb{E}_{S_{n+1}\sim p(\cdot|S_n)}\Big[r(S_n)+\gamma v_L(S_{n+1}; \theta_*) - v_L(S_n; \theta_*)\Big]}\notag \\
={}& \abs{\big(r + \gamma P_\pi v_L - v_L\big)(S_n)}\notag \\
={}& \abs{\big( (I - \gamma P_\pi)(v_\pi - v_L) \big)(S_n)}
\quad \text{(using } r = (I-\gamma P_\pi)v_\pi\text{)}\notag \\
\leq{}& \|e_q^\top(I-\gamma P_\pi)\|_1 \|v_L - v_\pi\|_\infty\notag \\
\leq{}& (1+\gamma)\|v_L - v_\pi\|_\infty.
\end{align}
By \eqref{eq:conv}, on the event of probability at least $1-\delta$
with $n \ge C_{\text{Thm}\tref{thm:finite-n},1}\log(1/\delta)$,
we have
\begin{align}
\abs{\E_{S_{n+1}\sim p(\cdot\mid S_n)}\big[\delta_L(\theta_*)\big]}
&\le
(1+\gamma)\left[
(1-C_{\ref{ass:diagonal-margin}})^{L}\|v_\pi\|_\infty
+ C_{\text{Thm}\tref{thm:finite-n},2}
\sqrt{\frac{\log(C_{\text{Thm}\tref{thm:finite-n},3}/\delta)}{n}}
\right],
\end{align}
where we use $v_0 = 0$.
Setting
$C_{\tref{lem:decay},1} \doteq C_{\text{Thm}\tref{thm:finite-n},1}$,
$C_{\tref{lem:decay},2} \doteq C_{\text{Thm}\tref{thm:finite-n},2}$,
and
$C_{\tref{lem:decay},3} \doteq C_{\text{Thm}\tref{thm:finite-n},3}$
completes the proof.
\end{proof}

% \begin{lemma}
% \label{lem:decay}
% Let $\theta \in \Theta^{\text{Looped}}$ be the shared parameters.
% Under Assumption~\ref{ass:diagonal-margin}, initialized with $v_0=0$, we have for any
% $(p,r)$ and any query state $S_n$ that
% \begin{equation}
% \label{eq:decay_td_cond}
% \abs{\E_{S_{n+1}\sim p(\cdot\mid S_n)}\big[\delta_L(\theta)\big]}
% \le (1+\gamma)\|v_\pi\|_\infty\, (1-C_{\text{\ref{ass:diagonal-margin}}})^{\,L}.
% \end{equation}
% \end{lemma}

% \begin{proof}
% By \eqref{eq:emerdelta} and the linearity of expectation we have
% \begin{align}
% \label{eq:decay_relation}
% &\abs{\mathbb{E}_{S_{n+1}\sim p(\cdot|S_n)}\Big[r(S_n)+\gamma v_L(S_{n+1}; \theta) - v_L(S_n; \theta)\Big]}\notag \\
% =& \abs{r(S_n) + \gamma \sum_{s' \in \mathcal{S}} p(s' \mid S_n) v_L(s'; \theta) - v_L(S_n; \theta)}\notag \\
% =& \abs{\big(r + \gamma P_\pi v_L - v_L\big)(S_n)}\notag \\
% =& \abs{\big( (I - \gamma P_\pi)v_\pi - (I - \gamma P_\pi)v_L \big)(S_n)} \quad \text{(using } r = (I-\gamma P_\pi)v_\pi\text{)}\notag \\
% =& \abs{\big( (I - \gamma P_\pi)(v_\pi - v_L) \big)(S_n)}\notag \\
% \leq& \|e_q^\top(I-\gamma P_\pi)\|_1 \|v_L - v_\pi\|_\infty\notag \\
% \leq& (1+\gamma)\|v_L - v_\pi\|_\infty\notag\\
% \leq& (1+\gamma)(1-C_{\text{\ref{ass:diagonal-margin}}})^L \|v_0 - v_\pi\|_\infty \quad\text{(By Assumption~\ref{ass:diagonal-margin})} \notag\\
% =& (1+\gamma)(1-C_{\text{\ref{ass:diagonal-margin}}})^L \|v_\pi\|_\infty.
% \end{align}
% This completes the proof.
% \end{proof}

\begin{lemma}
\label{lem:softmax-col-l1}
For any $G, G' \in \R[(n+1) \times (n+1)]$,
\begin{equation}
\label{eq:lip}
    \max_{j \in [n+1]} 
  \bigl\| \bigl[\softmax(G) - \softmax(G')\bigr][:,j] \bigr\|_1
  \le 2 \max_{i,j} |G[i,j] - G'[i,j]|.
\end{equation}
\end{lemma}

\begin{proof}
Fix a column $j \in [n+1]$.
By \eqref{eq:kermatrix}, the masked softmax acts on the
first $n$ entries of column $j$ via the standard softmax map
$\sigma: \R[n] \to \R[n]$ defined by
$\sigma_i(w) = \frac{\exp(w_i)}{\sum_{k=1}^n \exp(w_k)}$,
while the $(n+1)$-th entry is fixed at zero.
 
The Jacobian $J(w) \in \R[n \times n]$ of $\sigma$ has entries
\begin{equation}
J_{ik}(w) = \frac{\partial \sigma_i(w)}{\partial w_k}
= \begin{cases}
\sigma_i(w)(1 - \sigma_i(w)), & k = i,\\
-\sigma_i(w)\,\sigma_k(w), & k \neq i.
\end{cases}
\end{equation}
For any direction $h \in \R[n]$, 
the directional derivative of $\sigma$ at $w$ along $h$ is
\[
  [J(w)\,h]_i 
  = \sigma_i(w)\,h_i - \sigma_i(w)\sum_{k=1}^n \sigma_k(w)\,h_k
  = \sigma_i(w)\Bigl(h_i - \sum_{k=1}^n \sigma_k(w)\,h_k\Bigr).
\]
Taking the $\ell_1$-norm and using $\sigma_i(w) \ge 0$, $\sum_i \sigma_i(w) = 1$,
\begin{align}
  \|J(w)\,h\|_1
  &= \sum_{i=1}^n \sigma_i(w)\,\Bigl|h_i - \sum_{k=1}^n \sigma_k(w)\,h_k\Bigr| \notag\\
  &\le \sum_{i=1}^n \sigma_i(w)\,\Bigl(|h_i| + \Bigl|\sum_{k=1}^n \sigma_k(w)\,h_k\Bigr|\Bigr) \notag\\
  &\le \|h\|_\infty + \|h\|_\infty\notag\\ 
  &= 2\,\|h\|_\infty, \label{eq:softmax-dir-l1}
\end{align}
where the second inequality is obtained since $\bigl|\sum_{k} \sigma_k(w)\,h_k\bigr| \le \|h\|_\infty$.
Hence $\sup_{\|h\|_\infty = 1} \|J(w)\,h\|_1 \le 2$ 
for all $w \in \R[n]$.
 
Let $w_j, w_j'$ denote the first $n$ entries 
of the $j$-th columns of $G$ and $G'$ respectively.
By the fundamental theorem of calculus,
\[
  \sigma(w_j) - \sigma(w_j')
  = \int_0^1 J\bigl(w_j' + t(w_j - w_j')\bigr)\,(w_j - w_j')\,dt.
\]
Taking the $\ell_1$-norm and applying \eqref{eq:softmax-dir-l1},
\[
  \|\sigma(w_j) - \sigma(w_j')\|_1
  \le \int_0^1 \bigl\|J\bigl(w_j' + t(w_j - w_j')\bigr)\,(w_j - w_j')\bigr\|_1\,dt
  \le 2\,\|w_j - w_j'\|_\infty.
\]
Since $\|w_j - w_j'\|_\infty \le \max_{i,j}|G[i,j] - G'[i,j]|$
and the $(n+1)$-th entry contributes zero to the $\ell_1$ difference,
taking the maximum over $j$ completes the proof.
\end{proof}

\subsection{Proof of Lemma~\tref{lemma:kernel_bound}}
\label{proof:kernel_bound}
\begin{lemma}
\label{lemma:kernel_bound}
Let $D_\theta \doteq d^2+3$ denote the number of free parameters
in \eqref{eq:parameterization_set}.
    There exists a constant $C_\tref{lemma:kernel_bound} > 0$ such that
    for any $\theta \in \Theta^{\mathrm{Looped}}$,
    \begin{equation}
        \max_{j \in [n+1]} 
        \bigl\| (\nabla_{\theta_m} \wK(\theta))[:,j] \bigr\|_1
        \le C_\tref{lemma:kernel_bound},
        \qquad \forall\, m = 1, \ldots, D_\theta.
    \end{equation}
\end{lemma}
\begin{proof}
Recall $C_x = \max_{s \in \fS}\|x(s)\|_2$
and define $C_R \doteq \max_{s}|r(s)|$.
Constructing $Z_0$ as in \eqref{eq:z_initial},
each column of $Z_0$ has at most $d+3$ entries,
so for all $i$
\begin{equation}
    \sum_{k=1}^{d+3} |Z_0[k,i]| \le (d+3)\max(C_x, C_R) \doteq C_{\tref{lemma:kernel_bound},1}
\end{equation}
Define the intermediate matrix
\[
    G(\theta) \doteq Z_0^\top A(\theta) Z_0,
    \qquad
    \wK(\theta) = \softmax\big(G(\theta)\big).
\]
For any perturbation $dA$, expanding $dG = Z_0^\top (dA) Z_0$ entrywise gives
\begin{align}
    |dG[i,j]|
    &= \Bigl|\sum_{k,l} Z_0[k,i]\,dA[k,l]\,Z_0[l,j]\Bigr|\\
    &\le \Bigl(\sum_k |Z_0[k,i]|\Bigr)
       \Bigl(\sum_l |Z_0[l,j]|\Bigr)
       \max_{k,l}|dA[k,l]|\\
    &\le C_{\tref{lemma:kernel_bound},1}^2\, \max_{k,l}|dA[k,l]|.
\end{align}
By \eqref{eq:lip},
\begin{equation}
    \max_j \|d\wK[:,j]\|_1
    \le 2\,\max_{i,j}|dG[i,j]|
    \le 2C_{\tref{lemma:kernel_bound},1}^2\max_{k,l}|dA[k,l]|.
\end{equation}
Finally, since $B$ is a sub-block of $A(\theta)$ by \eqref{eq:parameterization_set},
$\max_{k,l}|(\nabla_{\theta_m} A)[k,l]| \le 1$ for all $m = 1, \ldots, D_\theta$.
Setting $C_\tref{lemma:kernel_bound} \doteq 2C_{\tref{lemma:kernel_bound},1}^2$
completes the proof.
\end{proof}

\subsection{Proof of Lemma~\tref{lemma:grad_bound}}
\label{proof:grad_bound}
\begin{lemma}
\label{lemma:grad_bound}
Under Assumption~\ref{ass:diagonal-margin}, for any $\delta\in(0,1)$, there exist constants
$C_{\tref{lemma:grad_bound},1},\;C_{\tref{lemma:grad_bound},2},\;C_{\tref{lemma:grad_bound},3},\;C_{\tref{lemma:grad_bound},4}>0$
such that if $n\ge C_{\tref{lemma:grad_bound},1}\log(1/\delta)$, then with probability at least $1-\delta$,
\begin{equation}
\label{eq:grad_bound}
    \|\nabla_\theta v_L(S_n;\theta_*)\|_\infty
\;\le\;
C_{\ref{lemma:grad_bound},2}
\;+\;
C_{\ref{lemma:grad_bound},3}
\sqrt{\frac{\log(C_{\ref{lemma:grad_bound},4}/\delta)}{n}}.
\end{equation}
\end{lemma}
\begin{proof}
By \eqref{eq:update}, the induced attention matrix $\wK(\theta)$ 
in \eqref{eq:kermatrix} depends on $\theta$ through $B$. 
Moreover, for all $\theta \in \Theta^{\text{Looped}}$, 
$\wK$ is invariant to layer $l$.
We first derive the one-step forward map for a general 
$\theta \in \Theta^{\mathrm{Looped}}$.
Recall from~\eqref{eq:parameterization_set} that the free parameters 
are $u \in \R[1\times 3]$ and $B \in \R[d\times d]$.
Replacing $[1,1,-1]$ with $u$ 
in~\eqref{eq:current_value_head_output},
the layerwise update generalizes to
\begin{equation}\label{eq:general_forward_map}
  F_\theta(v)(s) 
  = v(s) + \sum_{k=1}^n 
    \bigl[u_1 R_k + u_2 \gamma v(S_k) + u_3 v(S_{k-1})\bigr]\,
    \wK(\theta)(S_{k-1}, s),
\end{equation}
where $\wK(\theta)$ is the attention matrix in~\eqref{eq:kermatrix}.
At $\theta_*$ with $u^* = [1,1,-1]$, 
$F_{\theta_*}(v) = \widehat{\fT}_{\wK}(v)$ 
as in~\eqref{eq:operator-compact}.

The Jacobian of $F_{\theta_*}$ with respect to $v$ is
\begin{equation}\label{eq:def_J}
  J \doteq \nabla_v F_{\theta_*}(v)
     = I - \wM_n + \gamma\widehat{P}_n,
\end{equation}
which is constant across layers.
By the chain rule,
\begin{equation}
\label{eq:grad_rec_constJ}
\nabla_\theta v_{l+1}
=
J\,\nabla_\theta v_l
+
g_l^{(B)} + g_l^{(u)},
\end{equation}
where $g_l^{(B)}$ and $g_l^{(u)}$ are the partial derivatives of $F_\theta(v_l)$ 
with respect to $B$ and $u$ at $\theta_*$, holding $v_l$ fixed.

By Lemma~\ref{lem:contraction}, choosing
$C_{\tref{lemma:grad_bound},1} \doteq C_\tref{lem:contraction}$,
then on the event of probability at least $1-\delta$,
\begin{equation}
\label{eq:J_contract}
\|J\|_\infty
=
\|I - \wM_n + \gamma \widehat{P}_n\|_\infty
\le 1 - C_{\ref{ass:diagonal-margin}} < 1.
\end{equation}
Unrolling~\eqref{eq:grad_rec_constJ} 
with $\nabla_\theta v_0=0$ gives
\begin{align}
\label{eq:grad_unroll_constJ}
\|\nabla_\theta v_L\|_\infty
= \norm{\sum_{k=0}^{L-1} J^{L-1-k} g_k}_\infty
\le
\sum_{k=0}^{L-1} (1 - C_{\ref{ass:diagonal-margin}})^{L-1-k}\|g_k\|_\infty.
\end{align}
It remains to bound $\|g_k\|_\infty = \max(\|g_k^{(B)}\|_\infty, \|g_k^{(u)}\|_\infty)$.

\paragraph{Bounding $g_k^{(B)}$.}
By~\eqref{eq:general_forward_map}, 
at $\theta_*$ the $B$-channel acts only through $\wK(\theta)$.
Applying H\"older's inequality and Lemma~\tref{lemma:kernel_bound},
\begin{equation}\label{eq:gk_B_bound}
  \|g_k^{(B)}\|_\infty
  \le C_{\ref{lemma:kernel_bound}} \max_{1\le i\le n} |\delta_i^{(v_k)}|
  \le C_{\ref{lemma:kernel_bound}} \bigl(C_R + (1+\gamma)\|v_k\|_\infty\bigr),
\end{equation}
where $\delta_i^{(v_k)} = R_i + \gamma v_k(S_i) - v_k(S_{i-1})$
is the TD error at transition $i$.

\paragraph{Bounding $g_k^{(u)}$.}
By~\eqref{eq:general_forward_map}, 
differentiating with respect to $u_1, u_2, u_3$ at $\theta_*$ gives
coefficients $R_i$, $\gamma v_k(S_i)$, and $v_k(S_{i-1})$ respectively.
Since $\sum_{i=1}^n \wK[i,s] = 1$ for each $s$,
\begin{equation}\label{eq:gk_u_bound}
  \|g_k^{(u)}\|_\infty 
  \le \max\bigl(C_R,\; \gamma\|v_k\|_\infty,\; \|v_k\|_\infty\bigr)
  \le C_R + (1+\gamma)\|v_k\|_\infty.
\end{equation}

\paragraph{Combining.}
By~\eqref{eq:gk_B_bound} and~\eqref{eq:gk_u_bound},
\begin{align}\label{eq:gk_bound_delta}
  \|g_k\|_\infty 
  &= \max\bigl(\|g_k^{(B)}\|_\infty,\, \|g_k^{(u)}\|_\infty\bigr)\\
  &\le \max(1, C_{\ref{lemma:kernel_bound}}) 
     \bigl(C_R + (1+\gamma)\|v_k\|_\infty\bigr)\notag\\
  &\leq C_{\ref{lemma:grad_bound},5}(\|v_k - v_\pi\|_\infty + \|v_\pi\|_\infty)\notag\\
  &\leq C_{\ref{lemma:grad_bound},6}+C_{\ref{lemma:grad_bound},7}
\sqrt{\frac{\log(C_{\text{Thm}\tref{thm:finite-n},3}/\delta)}{n}},
\end{align}
where the last step applies Theorem~\tref{thm:finite-n} with $\theta = \theta_*$
on the event of probability at least $1-\delta$ 
with $n \ge C_{\text{Thm}\tref{thm:finite-n},1}\log(1/\delta)$.
Crucially, the right-hand side is independent of the layer index $k$.
Substituting into~\eqref{eq:grad_unroll_constJ} gives
\begin{align}
  \|\nabla_\theta v_L\|_\infty
  &\le
  \frac{1}{C_{\ref{ass:diagonal-margin}}}
  \qty(C_{\ref{lemma:grad_bound},6}+C_{\ref{lemma:grad_bound},7}
\sqrt{\frac{\log(C_{\text{Thm}\tref{thm:finite-n},3}/\delta)}{n}}).
  \label{eq:grad_final}
\end{align}
Setting
$C_{\ref{lemma:grad_bound},1}
   \doteq C_{\text{Thm}\ref{thm:finite-n},1}$,
$C_{\ref{lemma:grad_bound},2}
   \doteq C_{\ref{lemma:grad_bound},6}
          / C_{\ref{ass:diagonal-margin}}$,
$C_{\ref{lemma:grad_bound},3}
   \doteq C_{\ref{lemma:grad_bound},7}
          / C_{\ref{ass:diagonal-margin}}$,
$C_{\ref{lemma:grad_bound},4}
   \doteq C_{\text{Thm}\tref{thm:finite-n},3}$
completes the proof.
\end{proof}

\subsection{Proof of Theorem~\tref{thm:emergence}}
\label{proof:emergence}

\begin{proof}
Let
\[
C_R \doteq \operatorname*{ess\,sup}_{(p,r)\sim\Delta}\|r\|_\infty,
\qquad
C_v \doteq \operatorname*{ess\,sup}_{(p,r)\sim\Delta}\|v_\pi\|_\infty
\leq \frac{C_R}{1-\gamma},
\qquad
D_\theta \doteq d^2+3.
\]

For $\eta\in(0,1/2]$, let $\mathcal E_\eta$ be the event on which Lemma~\tref{lem:decay} and Lemma~\tref{lemma:grad_bound} both hold with confidence parameter $\eta/2$. Define
\[
C_{\text{Thm}\tref{thm:emergence},4}
\doteq 2\max\qty{C_{\tref{lem:decay},1},\,C_{\tref{lemma:grad_bound},1}},
\qquad
C_{\text{Thm}\tref{thm:emergence},5}
\doteq 2\max\qty{e,\,C_{\tref{lem:decay},3},\,C_{\tref{lemma:grad_bound},4}}.
\]
By a union bound, $\Pr(\mathcal E_\eta)\geq 1-\eta$ whenever $n\geq C_{\text{Thm}\tref{thm:emergence},4}\log(1/\eta)$, and with this choice of $C_{\text{Thm}\tref{thm:emergence},5}$, the logarithmic terms in both lemmas are bounded by $\log(C_{\text{Thm}\tref{thm:emergence},5}/\eta)$.

Since $\delta_L(\theta_*)$ in~\eqref{eq:emerdelta} depends on $S_{n+1}$, define
\[
\bar\delta_L(\theta_*)
\doteq
\E_{S_{n+1}\sim p(\cdot|S_n)}
\qty[\delta_L(\theta_*)]
=
r(S_n)+\gamma(P_\pi v_L)(S_n)-v_L(S_n).
\]
By the tower property,
\[
\E_{(p,r),\tau_{n+1}}
\qty[\delta_L(\theta_*)\nabla_\theta v_L(S_n;\theta_*)]
=
\E_{(p,r),\tau_n}
\qty[\bar\delta_L(\theta_*)\nabla_\theta v_L(S_n;\theta_*)].
\]
Let $X_L\doteq \bar\delta_L(\theta_*)\nabla_\theta v_L(S_n;\theta_*)$. Decomposing the expectation over $\mathcal E_\eta$ and $\mathcal E_\eta^c$,
\begin{equation}
\label{eq:JL_event_split}
J_L(\theta_*)
=
\norm{\E[X_L]}_1
\leq
\E\qty[\norm{X_L}_1\mathbf 1_{\mathcal E_\eta}]
+
\E\qty[\norm{X_L}_1\mathbf 1_{\mathcal E_\eta^c}].
\end{equation}

On $\mathcal E_\eta$, Lemma~\tref{lem:decay} and Lemma~\tref{lemma:grad_bound} give, respectively,
\begin{align}
\abs{\bar\delta_L(\theta_*)}
&\leq
(1+\gamma)
\qty[
C_v(1-C_{\ref{ass:diagonal-margin}})^L
+
C_{\tref{lem:decay},2}
\sqrt{
\frac{\log(C_{\text{Thm}\tref{thm:emergence},5}/\eta)}{n}
}
]
\notag\\
&\leq
C_{\text{Thm}\tref{thm:emergence},6}
\qty[
(1-C_{\ref{ass:diagonal-margin}})^L
+
\sqrt{\frac{\log(C_{\text{Thm}\tref{thm:emergence},5}/\eta)}{n}}
],\\[4pt]
\norm{\nabla_\theta v_L(S_n;\theta_*)}_\infty
&\leq
C_{\tref{lemma:grad_bound},2}
+
C_{\tref{lemma:grad_bound},3}
\sqrt{\frac{\log(C_{\text{Thm}\tref{thm:emergence},5}/\eta)}{n}}
\notag\\
&\leq
C_{\text{Thm}\tref{thm:emergence},7}
\qty[
1+
\sqrt{\frac{\log(C_{\text{Thm}\tref{thm:emergence},5}/\eta)}{n}}
],
\end{align}
where
\[
C_{\text{Thm}\tref{thm:emergence},6}
\doteq (1+\gamma)\max\qty{C_v,\,C_{\tref{lem:decay},2}},
\qquad
C_{\text{Thm}\tref{thm:emergence},7}
\doteq \max\qty{C_{\tref{lemma:grad_bound},2},\,C_{\tref{lemma:grad_bound},3}}.
\]
on $\mathcal E_\eta$, by H\"older's inequality, with $C_{\text{Thm}\tref{thm:emergence},8}\doteq D_\theta\,C_{\text{Thm}\tref{thm:emergence},6}\,C_{\text{Thm}\tref{thm:emergence},7}$,
\begin{equation}
\label{eq:X_event_bound}
\norm{X_L}_1
\leq
C_{\text{Thm}\tref{thm:emergence},8}
\qty[
(1-C_{\ref{ass:diagonal-margin}})^L
+
\sqrt{\frac{\log(C_{\text{Thm}\tref{thm:emergence},5}/\eta)}{n}}
]
\qty[
1+
\sqrt{\frac{\log(C_{\text{Thm}\tref{thm:emergence},5}/\eta)}{n}}
].
\end{equation}

It remains to control $\norm{X_L}_1$ deterministically. Let $G_n\doteq I-M_n^c+\gamma\widehat P_n$. Since $M_n^c$ and $\widehat P_n$ are nonnegative row-stochastic, $\norm{G_n}_\infty\leq 2+\gamma$. Applying this to the recursion~\eqref{eq:operator-compact} with $v_0=0$,
\begin{align}
\norm{v_{l+1}}_\infty
&\leq (2+\gamma)\norm{v_l}_\infty+C_R
\quad\Longrightarrow\quad
\norm{v_L}_\infty\leq \frac{C_R\qty((2+\gamma)^L-1)}{1+\gamma},\\[4pt]
\abs{\bar\delta_L(\theta_*)}
&\leq C_R+(1+\gamma)\norm{v_L}_\infty
\leq C_R(2+\gamma)^L.
\label{eq:delta_all_bound}
\end{align}
For the gradient, from the proof of Lemma~\tref{lemma:grad_bound} there exists a constant $C_{\text{Thm}\tref{thm:emergence},9}>0$ such that 
\begin{equation}\norm{g_l}_\infty\leq C_{\text{Thm}\tref{thm:emergence},9}(C_R+(1+\gamma)\norm{v_l}_\infty)\leq C_{\text{Thm}\tref{thm:emergence},9}C_R(2+\gamma)^l.
\end{equation}
Unrolling $\nabla_\theta v_{l+1}=G_n\nabla_\theta v_l+g_l$ with $\nabla_\theta v_0=0$,
\begin{equation}
\label{eq:grad_all_bound}
\norm{\nabla_\theta v_L}_\infty
\leq
\sum_{l=0}^{L-1}(2+\gamma)^{L-1-l}\norm{g_l}_\infty
\leq
C_{\text{Thm}\tref{thm:emergence},9}C_R\,L(2+\gamma)^{L-1}.
\end{equation}
Combining~\eqref{eq:delta_all_bound} and~\eqref{eq:grad_all_bound}, and defining $C_{\text{Thm}\tref{thm:emergence},10}\doteq D_\theta\,C_{\text{Thm}\tref{thm:emergence},9}\,C_R^2$,
\begin{equation}
\label{eq:X_all_bound}
\norm{X_L}_1
\leq
C_{\text{Thm}\tref{thm:emergence},10}\,L(2+\gamma)^{2L}
\qquad\text{for every trajectory.}
\end{equation}

Substituting~\eqref{eq:X_event_bound} and~\eqref{eq:X_all_bound} into~\eqref{eq:JL_event_split} and setting $\eta=1/n$,
\begin{align}
J_L(\theta_*)
&\leq
C_{\text{Thm}\tref{thm:emergence},8}
\qty[
(1-C_{\ref{ass:diagonal-margin}})^L
+
\sqrt{\frac{\log(C_{\text{Thm}\tref{thm:emergence},5}n)}{n}}
]
\qty[
1+
\sqrt{\frac{\log(C_{\text{Thm}\tref{thm:emergence},5}n)}{n}}
]
\notag\\
&\qquad+
C_{\text{Thm}\tref{thm:emergence},10}
\frac{L(2+\gamma)^{2L}}{n}.
\label{eq:JL_eta_bound}
\end{align}
Define $C_{\text{Thm}\tref{thm:emergence},11}\doteq 1+\log C_{\text{Thm}\tref{thm:emergence},5}/\log 2$. Since $C_{\text{Thm}\tref{thm:emergence},5}\geq e$, for every $n\geq 2$, $\log(C_{\text{Thm}\tref{thm:emergence},5}n)\leq C_{\text{Thm}\tref{thm:emergence},11}\log n$. Choose $n_0$ so that for every $n\geq n_0$,
\[
n\geq C_{\text{Thm}\tref{thm:emergence},4}\log n
\quad\text{and}\quad
\sqrt{\frac{\log(C_{\text{Thm}\tref{thm:emergence},5}n)}{n}}\leq 1.
\]
Using $(a+b)(1+b)\leq a+3b$ for $a,b\in[0,1]$,
\[
\qty[
(1-C_{\ref{ass:diagonal-margin}})^L
+
\sqrt{\frac{\log(C_{\text{Thm}\tref{thm:emergence},5}n)}{n}}
]
\qty[
1+
\sqrt{\frac{\log(C_{\text{Thm}\tref{thm:emergence},5}n)}{n}}
]
\leq
(1-C_{\ref{ass:diagonal-margin}})^L
+
3\sqrt{C_{\text{Thm}\tref{thm:emergence},11}}\sqrt{\frac{\log n}{n}}.
\]
Therefore, with
\[
C_{\text{Thm}\tref{thm:emergence},2}
\doteq C_{\text{Thm}\tref{thm:emergence},8}
\max\qty(1,\,3\sqrt{C_{\text{Thm}\tref{thm:emergence},11}}),
\qquad
C_{\text{Thm}\tref{thm:emergence},3}
\doteq C_{\text{Thm}\tref{thm:emergence},10},
\]
we obtain~\eqref{eq:emer}: for all $n\geq n_0$,
\[
J_L(\theta_*)
\leq
C_{\text{Thm}\tref{thm:emergence},2}
\qty[
(1-C_{\ref{ass:diagonal-margin}})^L
+
\sqrt{\frac{\log n}{n}}
]
+
C_{\text{Thm}\tref{thm:emergence},3}
\frac{L(2+\gamma)^{2L}}{n}.
\]

Finally, for every fixed $L$, $\limsup_{n\to\infty}J_L(\theta_*)\leq C_{\text{Thm}\tref{thm:emergence},2}(1-C_{\ref{ass:diagonal-margin}})^L$, and taking $L\to\infty$ gives $\lim_{L\to\infty}\limsup_{n\to\infty}J_L(\theta_*)=0$.
\end{proof}

\section{Additional Experimental Details}
\label{sec:aux_experiment}
\subsection{Multi-task TD}
The pseudocode of multi-task TD is provided in Algorithm~1.
Notably, $\text{TF}_L$ is obtained by~\eqref{eq:output}.
The training update is semi-gradient TD, where
the bootstrap target $\text{TF}_L(Z_0';\theta)$ is held fixed,
so the single-sample update direction is $\delta_t \nabla_\theta \text{TF}_L(Z_0;\theta)$,
which is a stochastic estimate of the expected update field
in the NEU objective~\eqref{eq:NEU}.

To accelerate training in some settings, we introduce a scalar diagonal mixing coefficient $\rho_t\in[0,1]$ at the attention level.
Let $L_l \doteq Z_l^\top A_l Z_l \in \mathbb{R}^{n\times n}$ and fix a softmax temperature $\tau>0$.
At training step $t$ the kernel is
\begin{equation}
\label{eq:rho}
    \wK[i,j](\rho_t)
    \doteq
    (1-\rho_t)\cdot\frac{\exp(L_{l,ij}/\tau)}{\sum_{m=0}^{n-1}\exp(L_{l,mj}/\tau)}
    + \rho_t\cdot \ind\{i=j\}.
\end{equation}
Importantly, all results in Section~\tref{sec:experiment} are trained with $\rho_t\equiv 0$.
\begin{algorithm}[t]
    \caption{\label{alg:ictd_softmax} Multi-Task Temporal Difference Learning with softmax Softmax Transformer}
    \begin{algorithmic}[1]
        \STATE \textbf{Input:}
        context length $n$,
        MRP number of states $m$,
        MRP sample length $H$,
        number of training tasks $k$,
        learning rate $\eta$,
        discount factor $\gamma$,
        diagonal mixing $\rho_t$,
        softmax temperature $\tau$,
        transformer parameters $\theta \doteq \qty(V, A)$

        \FOR{$i \gets 1$ \textbf{to} $k$}
            \STATE Sample $(p_0, p, r, x)$ from $d_{\text{task}}$
            \STATE Sample $(S_0, R_1, S_1, R_2, \dots, S_{H}, R_{H+1}, S_{H+1})$ from the MRP $(p_0, p, r)$
            \FOR{$t = 0, \dots, H - n - 1$}
                \STATE // Build two consecutive prompt windows
                \STATE // Query uses the state after the last context transition: $x_q \doteq x_{t+n}$, $x_q' \doteq x_{t+n+1}$
                \STATE $Z_0 \gets \mqty[x_{t} & \cdots & x_{t+n-1} & x_{q} \\
                             R_{t+1} & \cdots & R_{t+n} & 0 \\
                             0 & \cdots & 0 & 0 \\
                             0 & \cdots & 0 & 0]$
                \STATE $Z_0' \gets \mqty[x_{t+1} & \cdots & x_{t+n} & x_{q}' \\
                             R_{t+2} & \cdots & R_{t+n+1} & 0 \\
                             0 & \cdots & 0 & 0 \\
                             0 & \cdots & 0 & 0]$

                \STATE // semi-gradient TD
                \STATE $\delta_t \gets R_{t+n+1} + \gamma\,\TF_L(Z_0';\theta) - \TF_L(Z_0;\theta)$
                \STATE $\theta \gets \theta + \eta \delta_t \nabla_\theta \TF_L(Z_0;\theta)$
            \ENDFOR
        \ENDFOR
    \end{algorithmic}
\end{algorithm}

\subsection{Minimal Sparse Parameterization}
\label{app:minimal-mask}

We use a minimal sparse parameterization that prevents the block from overwriting features or rewards,
while leaving the structured computation to be learned from data.
Notably, this sparsification is strictly lighter than the sparse parameter space used in Section~\tref{sec:emergence}.
The masks below are fixed throughout training, and all trainable entries are initialized by Xavier normal initialization with gain $0.1$.

\paragraph{Mask on $V$.}
Recall that $Z_l\in\R[(d+3)\times(n+1)]$ defined in \eqref{eq:z_general}.
We enforce a fixed sparsity pattern via an elementwise mask
$V_{\mathrm{mask}}\in\{0,1\}^{(d+3)\times(d+3)}$ and parameterize
\begin{equation}
\label{eq:v0}
    V_0 = \widetilde V \odot V_{\mathrm{mask}},
\qquad
\widetilde V\in\R[(d+3)\times(d+3)].
\end{equation}
The mask is fixed throughout training and only allows the last two rows to be nonzero:
\begin{equation}
\label{eq:vmask}
    V_{\mathrm{mask}} =
\begin{bmatrix}
0_{(d+1)\times(d+3)}\\
\ind_{2\times(d+3)}
\end{bmatrix}.
\end{equation}
As a result, the additive update $V_0 Z_l$ modifies only the two value slots of $Z_l$.

\paragraph{Mask on $A$.}
Let $A_0\in\R[(d+3)\times(d+3)]$ be the logits mixing matrix used to form the pre-softmax logits
$Z_l^\top A_0 Z_l$. We impose the block sparsity via an elementwise mask
$A_{\mathrm{mask}}\in\{0,1\}^{(d+3)\times(d+3)}$ and parameterize
\begin{equation}
\label{eq:a0}
    A_0 = \widetilde A \odot A_{\mathrm{mask}},
\qquad
\widetilde A\in\R[(d+3)\times(d+3)].
\end{equation}
We split rows and columns into the first $d$ feature channels and the last $3$ scalar channels, and set the
scalar-to-scalar block to zero by
\begin{equation}
\label{eq:amask}
    A_{\mathrm{mask}}=
\begin{bmatrix}
\ind_{d\times d} & \mathbf{0}_{d\times 3}\\
\mathbf{0}_{3\times d} & \mathbf{0}_{3\times 3}
\end{bmatrix}.
\end{equation}

% Equivalently, viewing $A_0$ as a $2\times 2$ block matrix,
% \[
% A_0=
% \begin{bmatrix}
% A^{\mathrm{ff}} & A^{\mathrm{fs}}\\
% A^{\mathrm{sf}} & 0_{3\times 3}
% \end{bmatrix},
% \]
% where $A^{\mathrm{ff}}\in\R[d\times d]$, $A^{\mathrm{fs}}\in\R[d\times 3]$,
% and $A^{\mathrm{sf}}\in\R[3\times d]$ are learnable.

% \paragraph{Mask on $A$.}
% Let $A\in\R[(d+3)\times(d+3)]$ be the logits mixing matrix used to form the pre-softmax logits
% $Z_l^\top A Z_l$.
% We view $A$ as a $2\times 2$ block matrix by splitting the rows and columns into the first $d$ feature channels
% and the last $3$ scalar channels. We set the scalar-to-scalar block to zero:
% \[
% A
% =
% \begin{bmatrix}
% A^{\mathrm{ff}} & A^{\mathrm{fs}}\\
% A^{\mathrm{sf}} & 0_{3\times 3}
% \end{bmatrix},
% \]
% where $A^{\mathrm{ff}}\in\R[d\times d]$, $A^{\mathrm{fs}}\in\R[d\times 3]$,
% and $A^{\mathrm{sf}}\in\R[3\times d]$ are trainable.

\subsection{Boyan's chain tasks and experiment setup}
\label{app:boyan-setup}

\begin{table}[t]
\centering
\caption{Hyperparameters for Boyan's chain experiments in Section~\tref{sec:experiment}.}
\label{tab:boyan-hparams}
\begin{tabular}{l c}
\hline
MRP (Boyan chain) states $m$ & $64$ \\
feature dimension $d$ & $4$ \\
discount factor $\gamma$ & $0.9$ \\
context length $n$ & $10$ \\
rollout segment length $H$ & $n+1=11$ \\
query feature & $x_q = x_{t+n}$\\
softmax temperature $\tau$ & $1.2$ \\
diagonal mixing $\rho$ & $0$ \\
\hline
optimizer & Adam \citep{kingma2014adam}\\
learning rate & $10^{-3}$ \\
weight decay & $0$ \\
batch size (contexts) & $64$ \\
mini-batches per epoch & $5$ \\
epochs & $3000$ \\
number of layers $L$ & $3$ \\
number of random seeds & $5$ \\
\hline
\end{tabular}
\end{table}

\begin{figure}[t]
    \centering
    \includegraphics[width=0.52\linewidth]{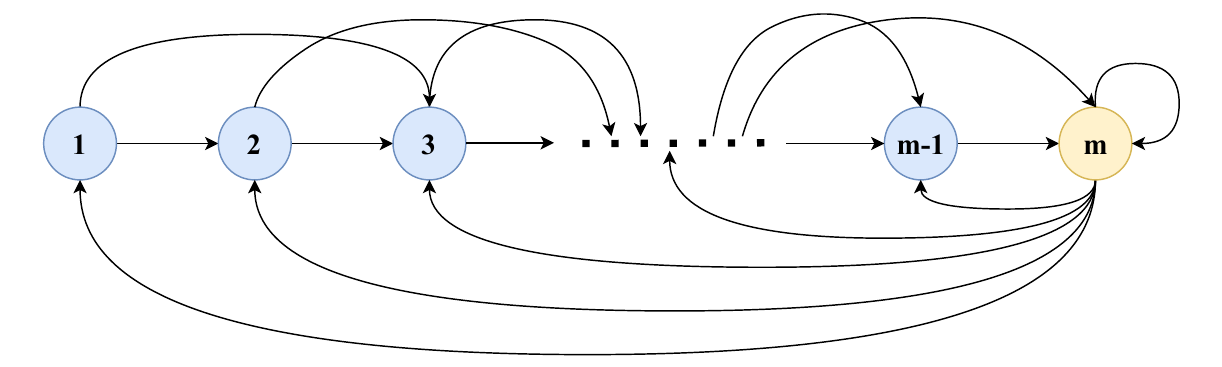}
    \caption{Boyan's chain topology with nonzero transitions. Adapted from \citet{wang2025ictd}.}
    \label{fig:boyan-topology}
\end{figure}

\paragraph{Task distribution.}
Following \citet{wang2025ictd}, we use a randomized variant of Boyan's chain~\citep{boyan1999least} as an evaluation-task distribution.
Each call to the generator samples an ergodic MRP together with a randomized state representation
$x:\mathcal S \to \R[d]$ and a ground-truth value function $v^*$ that is linear in $x(s)$.
The generation procedure is summarized in Algorithm~\ref{alg:representable}.
% The model is not given the task parameters that define $v^*$ (e.g., the sampled weight vector) and only receives contexts constructed from sampled trajectories.

\begin{algorithm}[t]
\caption{Boyan Chain MRP and Feature Generation (Adapted from Algorithm~3 of \citet{wang2025ictd})}
\begin{algorithmic}[1]
\label{alg:representable}
\STATE \textbf{Input:} state space size $m=\abs{\fS}$, feature dimension $d$, discount factor $\gamma$
\STATE $w_* \sim \mathrm{Uniform}\qty[(-1,1)^d]$\qquad // ground-truth weight
\FOR{$s \in \fS$}
  \STATE $x(s) \sim \mathrm{Uniform}\qty[(-1,1)^d]$\qquad // feature
  \STATE $v^*(s) \gets \langle w_*, x(s)\rangle$\qquad // ground-truth value function
\ENDFOR
\STATE $p_0 \sim \mathrm{Uniform}\qty[(0,1)^m]$ \qquad // initial distribution
\STATE $p_0 \gets p_0/\sum_s p_0(s)$
\STATE $p \gets 0_{m\times m}$ \qquad // transition function
\FOR{$i \gets 1$ to $m-2$}
    \STATE $\epsilon \sim \mathrm{Uniform}[(0,1)]$
    \STATE $p(i,i+1)\gets \epsilon$ 
    \STATE $p(i,i+2)\gets 1-\epsilon$
\ENDFOR
\STATE $p(m-1,m)\gets 1$
\STATE $z \sim \mathrm{Uniform}\big[(0,1)^m\big]$
\STATE $z \gets z/\sum_{s} z(s)$
\STATE $p(m,1:m)\gets z$
\STATE $r \gets (I_m-\gamma p)v^*$\qquad // reward function
\STATE Define rewards by Bellman consistency: $r \gets (I - \gamma P) v^*$.
\STATE \textbf{Output:} MRP $(p_0,p,r)$ and feature map $x$
\end{algorithmic}
\end{algorithm}

% \begin{algorithm}[t]
% \caption{Boyan Chain MRP and Feature Generation (Non-Presentable\footnote{Adapted from Algorithm~2 of \citet{wang2025ictd}})}
% \label{alg:boyan-chain}
% \begin{algorithmic}[1]
% \STATE \textbf{Input:} state space size $m=|\mathcal{S}|$, feature dimension $d$
% \STATE $r \sim \mathrm{Uniform}\big[(-1,1)^m\big]$ \qquad // reward function 
% \STATE $p_0 \sim \mathrm{Uniform}\qty[(0,1)^m]$ \qquad // initial distribution
% \STATE $p_0 \gets p_0/\sum_s p_0(s)$
% \STATE $p \gets 0_{m\times m}$ \qquad // transition function
% \FOR{$i \gets 1$ to $m-2$}
%     \STATE $\epsilon \sim \mathrm{Uniform}[(0,1)]$
%     \STATE $p(i,i+1)\gets \epsilon$ 
%     \STATE $p(i,i+2)\gets 1-\epsilon$
% \ENDFOR
% \STATE $p(m-1,m)\gets 1$
% \STATE $z \sim \mathrm{Uniform}\big[(0,1)^m\big]$
% \STATE $z \gets z/\sum_{s} z(s)$
% \STATE $p(m,1:m)\gets z$
% \FOR{each $s\in\mathcal{S}$}
%     \STATE $x(s)\sim \mathrm{Uniform}\big[(-1,1)^d\big]$ \qquad // feature map
% \ENDFOR
% \STATE \textbf{Output:} MRP $(p_0,p,r)$ and feature map $x$
% \end{algorithmic}
% \end{algorithm}

\paragraph{Training pipeline.}
Unless otherwise stated, we set the context length $n=10$, feature dimension $d=4$, and discount factor $\gamma=0.9$.
At the beginning of each epoch, we reset the MRP by sampling a fresh task (Algorithm~\ref{alg:representable}),
roll out a single on-policy trajectory, and form overlapping length-$n$ contexts by sliding a window along the trajectory.
We train a single-head softmax ICTD block with $L=3$ layers under the minimal sparse parameterization
(Appendix~\ref{app:minimal-mask}).
Unless otherwise stated, we train for $3000$ epochs with $5$ mini-batches per epoch and a mini-batch size of $64$ contexts.
Each experiment is repeated with $5$ random seeds and we aggregate results across seeds.

\paragraph{Hyperparameters.}
Table~\ref{tab:boyan-hparams} summarizes the default hyperparameters used in Boyan's chain experiments in Section~\tref{sec:experiment}.

\paragraph{Plotting details for Figures~\tref{fig:emergence-main}.}
For each random seed, we generate $N_{\mathrm{MRP}}=3000$ independent MRPs for evaluation and record diagnostics on them. Each evaluated point includes $d_t$ in \eqref{eq:kernel-diag-mean} and related quantities computed from the same evaluation rollouts. 
Emergence scores $V_{\mathrm{em}}$ and $A_{\mathrm{em}}$ are computed once per training step from the learned $(V_0,A_0)$.

\noindent\textbf{Figure~\tref{fig:emergence-main}(a).}
We select the best checkpoint as the training step that maximizes
$\min(V_{\mathrm{em}}, A_{\mathrm{em}})$, i.e., where both $V$ and $A$ emergence
are simultaneously high. At this checkpoint, we save the learned matrices
$V_0$ and $A_0$ as in \eqref{eq:v0} and \eqref{eq:a0}.
For visualization, each heatmap is normalized by its maximum absolute value and
rendered with a colormap over the range $[-1, 1]$.

\noindent\textbf{Figure~\tref{fig:emergence-main}(b).}
We bin evaluated points by $d_t$ using $20$ equal-width bins.
Within each bin, we plot the mean emergence scores $V_{\mathrm{em}}$ and $A_{\mathrm{em}}$ separately.
Shaded regions correspond to the interquartile range (25th-75th percentile) across random seeds.

\paragraph{Emergence vs.\ training step.}
For each seed, we smooth the per-step traces of $V_{\mathrm{em}}$, $A_{\mathrm{em}}$, and $d_t$ with a trailing average of window $50$.
We then truncate the curves at the seed-wise early-stop step that maximizes $\min(V_{\mathrm{em}},A_{\mathrm{em}})$.
Finally, we plot the mean $\pm$ one s.e.m.\ across seeds (top: $V_{\mathrm{em}}$ and $A_{\mathrm{em}}$; bottom: $d_t$).
Figure~\tref{fig:emer-vs-step-rho0} shows that under the default hyperparameters in Table~\tref{tab:boyan-hparams}, the model automatically learns increasingly diagonal kernels (high $d_t$) over training.

\begin{figure}[t]\centering
  \includegraphics[width=0.6\linewidth]{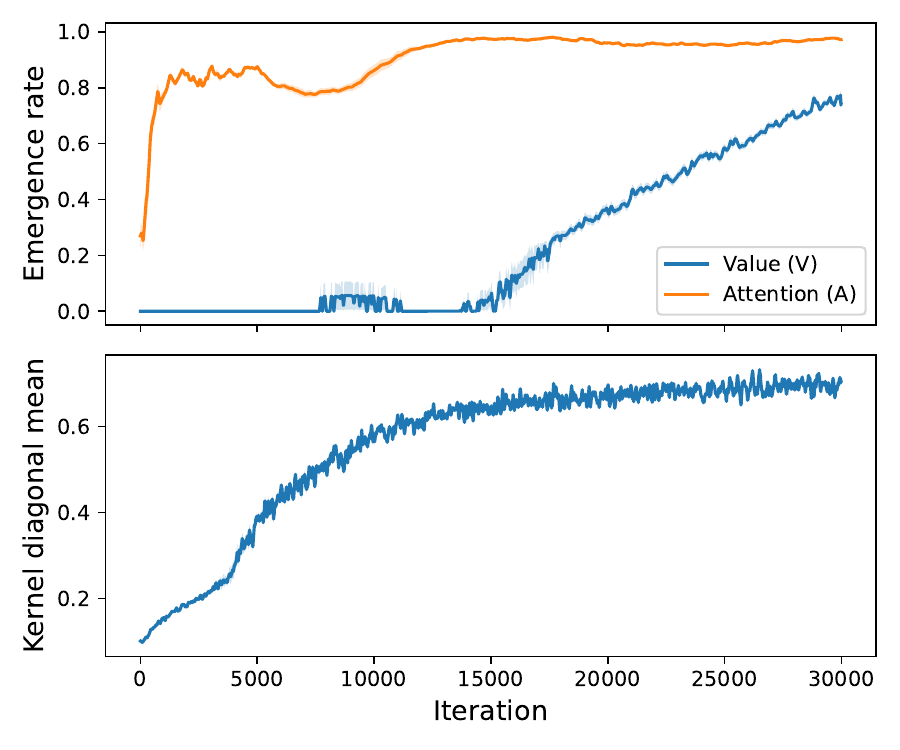}
  \caption{Emergence vs.\ training steps under the default mask and $\rho_t\equiv 0$.}
  \label{fig:emer-vs-step-rho0}
\end{figure}

% \paragraph{Compute resources.}
% We run our experiments in parallel on a single node of a CPU cluster. The node uses an AMD EPYC 9534 processor with 64 physical cores and 128 hardware threads.

\paragraph{Compute resources.}
All experiments were run on a single NVIDIA H100 GPU. The main experiments reported in Figure~\tref{fig:emergence-main} take approximately one hour of wall-clock time in total.
The appendix ablations, including the mask-relaxation experiments in Figure~\tref{fig:mask-relaxation-boyan}, use the same single-H100 setup and account for the majority of the remaining compute. 
The numerical verification experiments in Figure~\tref{fig:nonlinear-td-verification} are lightweight.
Across all reported experiments in the paper, the total compute is below 10 H100 GPU-hours. 
Preliminary and debugging runs used comparable compute and did not substantially exceed this reported experimental budget. 
The experiments use only synthetic Boyan-chain data, and the storage requirement is negligible.

\subsection{Emergence metrics}
\label{sec:emergence-metrics}

We quantify how closely the first softmax ICTD layer matches the ideal TD update
using two scalar scores $V_{\mathrm{em},t}$ and $A_{\mathrm{em},t}$ in $[0,1]$.
A score of $1$ corresponds to the ideal TD block.

\subsubsection{V-emergence score.}
At training step $t$, let
$(p_{r,t}, p_{\gamma v',t}, p_{v,t})
\doteq
(V_{0,t})_{d+3,d{:}d+3}$
denote the three coefficients that the first layer applies to the reward $U$,
bootstrap value $\gamma v'$, and current value $v$ in its value update.

\textbf{Sign check.}
We enforce the TD sign pattern via
\begin{equation}
\label{eq:sign}
\mathbb{I}_{V,t}
\doteq
\mathbb{I}\bigl\{p_{r,t}>0,\;p_{\gamma v',t}>0,\;p_{v,t}<0\bigr\}.
\end{equation}

\textbf{Comparability.}
We encourage comparable magnitudes.
Let
\[
m_{V,t}
\doteq
\frac{|p_{r,t}|+|p_{\gamma v',t}|+|p_{v,t}|}{3}.
\]
We define the magnitude-ratio score
\begin{equation}
\label{eq:ratio}
C_{V,t}
\doteq
\max\Biggl\{
0,\;
1-\frac{\bigl||p_{r,t}|-m_{V,t}\bigr|
        +\bigl||p_{\gamma v',t}|-m_{V,t}\bigr|
        +\bigl||p_{v,t}|-m_{V,t}\bigr|}{3\max\{m_{V,t},\varepsilon\}}
\Biggr\}
\in [0,1],
\end{equation}
where $\varepsilon>0$ is a small constant.
This score is maximized when the three coefficients have the same magnitude. When $m_{V,t}=0$
the definition yields $C_{V,t}=0$.

The V-emergence score at step $t$ is
\[
V_{\mathrm{em},t}\doteq \mathbb{I}_{V,t}\, C_{V,t}\in[0,1].
\]

\subsubsection{A-emergence score.}
Let $A_{0,t}^{\mathrm{feat}} \doteq (A_{0,t})_{0{:}d,\,0{:}d}\in \R[d\times d]$ at training step $t$.
For diagnostics we work with a column-normalized nonnegative version.
For $j\in[d]$, define
\[
(\widehat A_t)[:,j]
\doteq
\begin{cases}
\frac{\bigl|(A_{0,t}^{\mathrm{feat}})[:,j]\bigr|}{\|(A_{0,t}^{\mathrm{feat}})[:,j]\|_2} & \|(A_{0,t}^{\mathrm{feat}})[:,j]\|_2>0,\\
0 & \text{otherwise},
\end{cases}
\]
that is, we divide each column by its $l_2$-norm and take absolute values.

% \textbf{Sign check.}
% Let $a_t \doteq \mathrm{diag}(A_{0,t}^{\mathrm{feat}})\in\R[d]$.
% We require all diagonal entries to be positive:
% \[
% \mathbb{I}_{A,t}\doteq \mathbb{I}\{(a_t)_i>0\ \ \forall i\in[d]\}.
% \]

\textbf{Diagonality.}
From $\widehat A_t$ we measure how much of the $l_1$ mass lies on the diagonal via
\[
A_{\mathrm{diag},t}
\doteq
\frac{\sum_{i=1}^d (\widehat A_t)_{ii}}
     {\max\Bigl\{\sum_{i,j=1}^d (\widehat A_t)[i,j],\,\varepsilon\Bigr\}}
\in[0,1],
\]
where $\varepsilon>0$ is a small constant.
This quantity is close to $1$ when $\widehat A_t$ is nearly diagonal.

\textbf{Comparability.}
Let $a_t \doteq \mathrm{diag}(A_{0,t}^{\mathrm{feat}})\in\R[d]$.
We encourage the diagonal entries to be comparable in magnitude.
Let
\[
m_{A,t}\doteq \frac{1}{d}\sum_{i=1}^d (a_t)_i,
\qquad
C_{A,t}
\doteq
\max\Biggl\{0,\;
1-\frac{\sum_{i=1}^d |(a_t)_i-m_{A,t}|}{d\,\max\{m_{A,t},\varepsilon\}}
\Biggr\}.
\]
% By construction, $C_{A,t}$ is large only when the diagonal entries are close to each other,
% and it vanishes when $m_{A,t}\le 0$.
By construction, $C_{A,t}$ measures whether the raw diagonal entries 
are close to a common value, but it does not separately enforce that 
this common value is positive.

The A-emergence score at step $t$ is
\begin{equation}
\label{eq:Aem-def}
A_{\mathrm{em},t}
\doteq
A_{\mathrm{diag},t}\cdot C_{A,t}
\in [0,1].
\end{equation}

% As a scalar measure of first-layer kernel diagonality at training step $t$,
% we use the kernel diagonal mean $d_t \in [0,1]$ defined in
% Eq.~\eqref{eq:kernel-diag-mean} in Section~\tref{sec:experiment}.
% To produce the emergence–rate curves in Figure~\ref{fig:emer-rate}, we
% record $(V_{\mathrm{em}}, A_{\mathrm{em}}, d_t)$ at regularly spaced
% checkpoints for each seed.
% We discretize $d_t$ into fixed bins on $[0,1]$ and, for each bin and each
% seed, compute the maximal $V_{\mathrm{em}}$ (respectively $A_{\mathrm{em}}$)
% among all checkpoints whose $d_t$ falls into that bin.
% The plotted value for each bin is the median of these per-seed maxima, with
% the interquartile range shown as a shaded band.
% The same metrics and binning procedure are used for both the Boyan and
% CartPole experiments reported in the main text and in
% Appendix~\ref{app:cartpole-setup}.
\begin{figure}[t]\centering
  % Row 1: training trajectories
  \begin{minipage}[t]{0.48\textwidth}\centering
    \includegraphics[width=\linewidth]{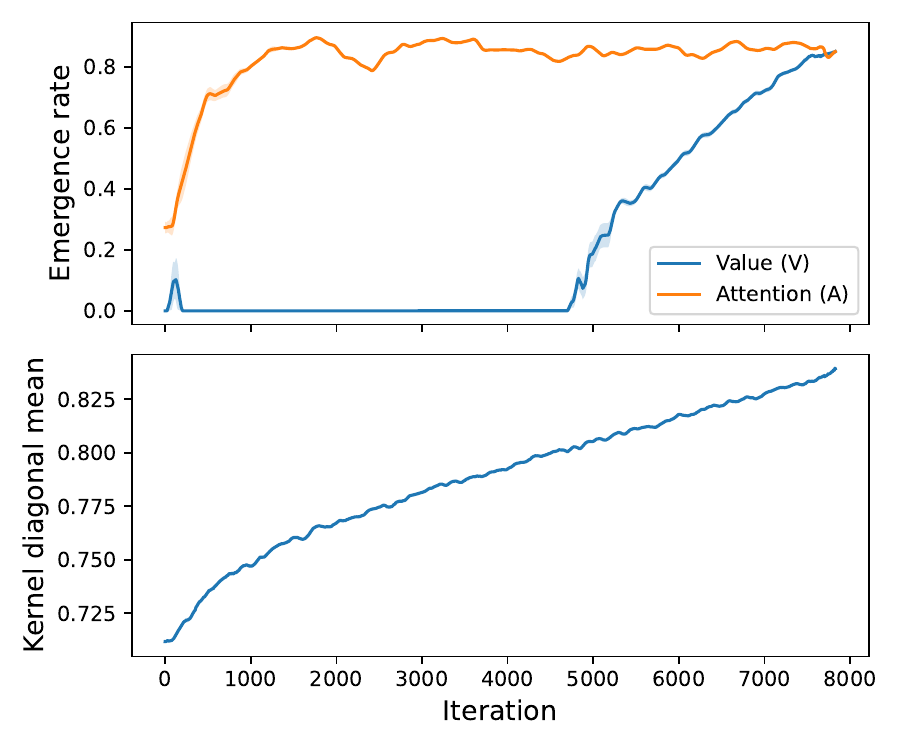}\\[2pt]
    \small (a) Full-mask: training trajectories.
  \end{minipage}
  \hspace{0.02\textwidth}
  \begin{minipage}[t]{0.48\textwidth}\centering
    \includegraphics[width=\linewidth]{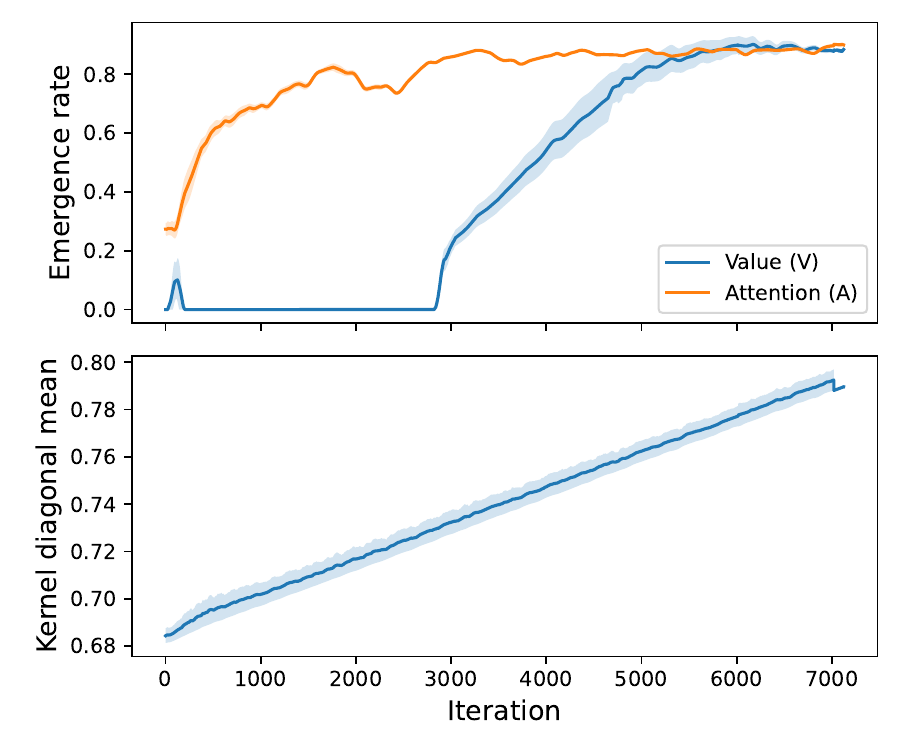}\\[2pt]
    \small (b) Relaxed-mask: training trajectories.
  \end{minipage}

  \vspace{0.6em}

  % Row 2: emergence rate vs diagonality
  \begin{minipage}[t]{0.48\textwidth}\centering
    \includegraphics[width=\linewidth]{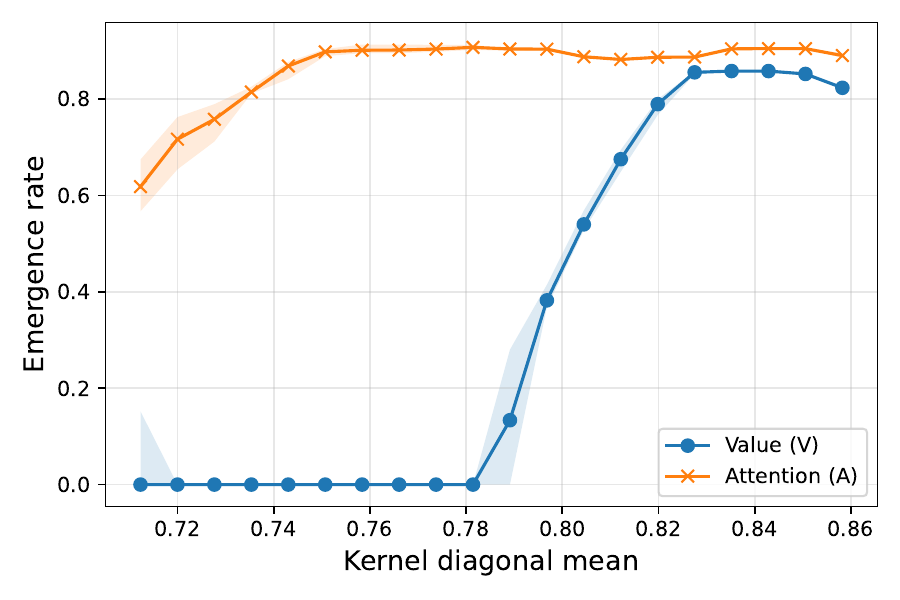}\\[2pt]
    \small (c) Full-mask: emergence vs.\ $d_t$.
  \end{minipage}
  \hspace{0.02\textwidth}
  \begin{minipage}[t]{0.48\textwidth}\centering
    \includegraphics[width=\linewidth]{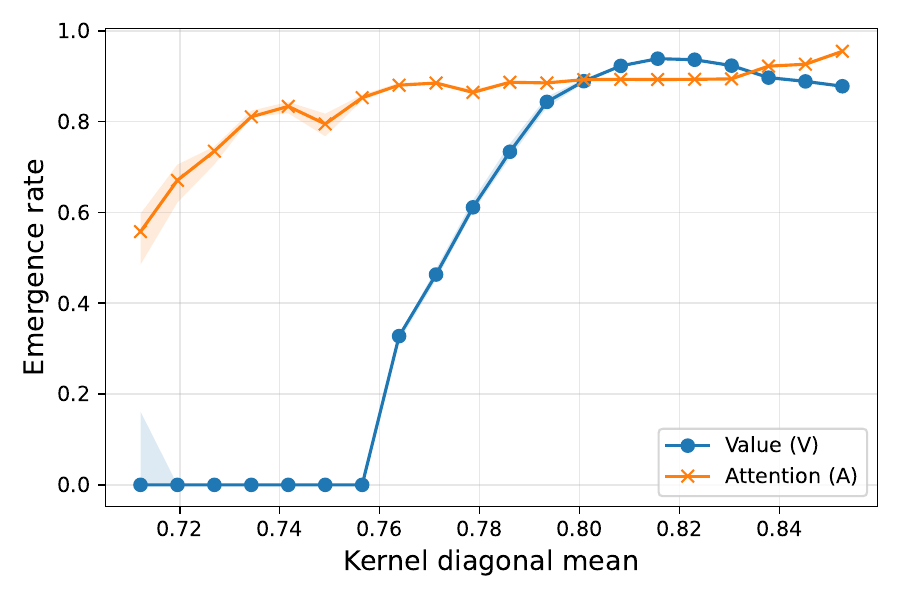}\\[2pt]
    \small (d) Relaxed-mask: emergence vs.\ $d_t$.
  \end{minipage}

  \vspace{0.6em}

  % Row 3: averaged V0/A0 heatmaps at best joint emergence
  \begin{minipage}[t]{0.48\textwidth}\centering
    \includegraphics[width=\linewidth]{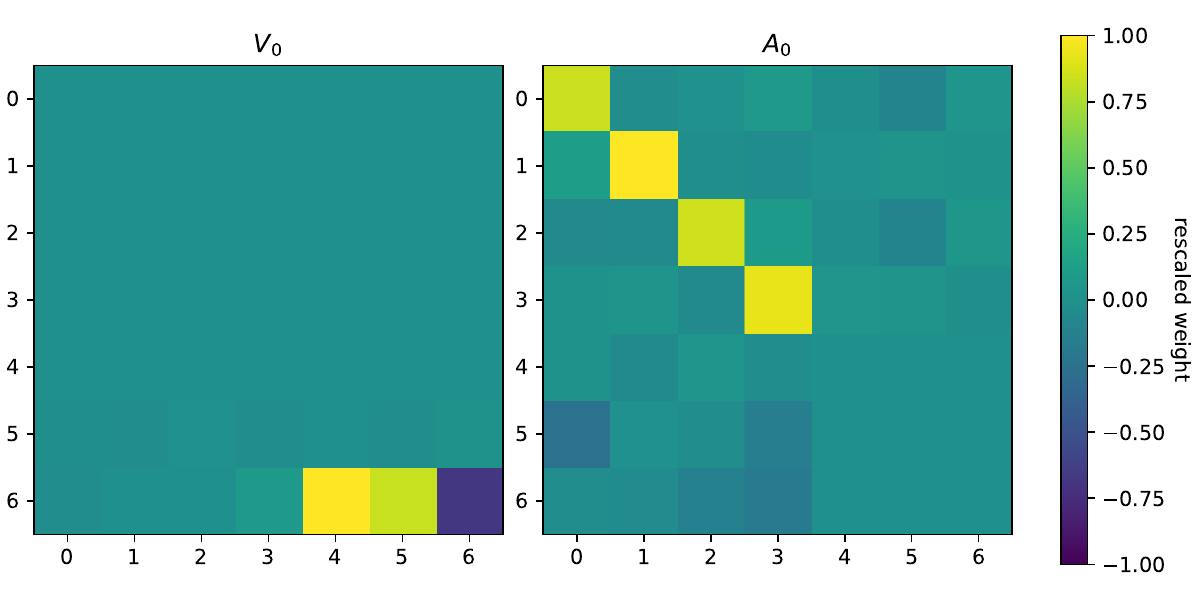}\\[2pt]
    \small (e) Full-mask: mean $V_0$ and $A_0$ at best joint emergence.
  \end{minipage}
  \hspace{0.02\textwidth}
  \begin{minipage}[t]{0.48\textwidth}\centering
    \includegraphics[width=\linewidth]{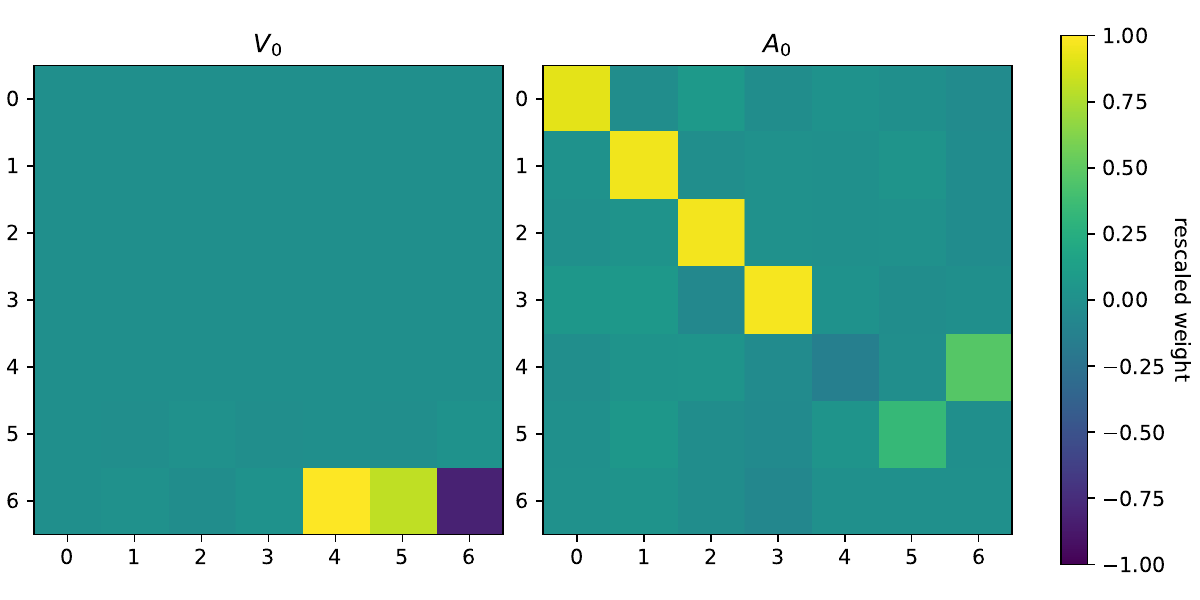}\\[2pt]
    \small (f) Relaxed-mask: mean $V_0$ and $A_0$ at best joint emergence.
  \end{minipage}
\caption{Mask relaxation on Boyan's chain.
We compare the full-mask setting in \eqref{eq:amask} (left column) and the relaxed-mask setting \eqref{eq:amaskrelax} (right column).
All plots are averaged and visualized in the same way as in Appendix~\ref{app:boyan-setup},
using $10$ random seeds over $9000$ training steps.
% Bottom row: mean $V_0$ and $A_0$ at the checkpoint that maximizes $\min(V_{\mathrm{em}},A_{\mathrm{em}})$.
}

%   \caption{Logits-mask relaxation on Boyan's chain.
% Top row: seed-averaged training trajectories of $V_{\mathrm{em},t}$ and $A_{\mathrm{em},t}$
% (top) and the kernel diagonal mean $d_t$ (bottom), shown as mean $\pm$ one s.e.m.\ over $10$ seeds
% (with the same temporal smoothing as in Section~\tref{sec:experiment}).
% Middle row: emergence rate as a function of $d_t$, obtained by binning $d_t\in[0,1]$ and, within each bin,
% taking the per-seed maximum of $V_{\mathrm{em},t}$ across checkpoints for the value-emergence curve,
% and the per-seed maximum of $A_{\mathrm{em},t}$ across checkpoints for the attention-emergence curve;
% we then plot the median across seeds with an interquartile-range band.
% Bottom row: mean $V_0$ and $A_0$ at the checkpoint with maximal joint $V/A$ emergence.
% }
  \label{fig:mask-relaxation-boyan}
\end{figure}

\subsection{Mask Relaxation}
\label{app:mask-relaxtion}
We next study how emergence depends on the strength of the structural masking constraint in $A_0$ on the Boyan task
distribution. 
All experiments in this subsection use the same Boyan training
setup as in Appendix~\ref{app:boyan-setup}, except that we apply the linearly
annealed diagonal mixing coefficient $\rho$ in \eqref{eq:rho}, increasing it from $0.68$
to $0.80$ over the first $8000$ steps and then holding it fixed at $0.80$.
It is trained for $9000$ gradient steps in total (corresponding to $900$ MRPs per seed in our implementation).

We introduce the diagonal mixing coefficient $\rho$ purely as a training-speed heuristic for this ablation:
under the default setting in Section~\tref{sec:experiment}, kernel diagonality emerges without any auxiliary bias given sufficient training,
whereas here we use $\rho$ to accelerate convergence so that the mask-relaxation comparison can be run with a smaller training budget
without changing the qualitative emergence behavior.
Throughout, we keep $V_{\mathrm{mask}}$ fixed.

We compare two masking configurations for $A_0$. The full-mask setting uses the
default $A_{\mathrm{mask}}$ from \eqref{eq:amask}. 
In the
relaxed-mask setting, we change it to
\begin{equation}
\label{eq:amaskrelax}
    A_{\mathrm{mask}}=
\begin{bmatrix}
\ind_{d\times d} & \ind_{d\times 3}\\
\ind_{3\times d} & \mathbf{0}_{3\times 3}
\end{bmatrix}.
\end{equation}
Figure~\ref{fig:mask-relaxation-boyan} summarizes the results and shows relaxing the mask preserves the emergence behavior.

\subsection{Numerical verification on Theorem~\tref{thm:onestep}}
We verified Theorem~\tref{thm:onestep} experimentally. For each trial we sample a fresh prompt $Z_0\in\mathbb{R}^{(d+3)\times(n+1)}$ and run the depth-$L$
forward recursion of a single-head block.
In parallel, we unroll the reference softmax ICTD recursion from Theorem~\tref{thm:onestep} on the same context,
producing $\{y_l^{(n+1)}\}_{l=0}^L$ for the query prediction.
We report the layer-wise log discrepancy
\[
E_l \;\doteq\; \log\bigl|v_l(S_n)-y_l^{(n+1)}\bigr|,
\qquad
v_l(S_n) \;=\; Z_l[d+3,n+1].
\]
We use $d=8$, $n=20$, $L=10$, $\gamma=0.9$, and $\alpha_l=1$, and repeat over $50$ trials in double precision.
We test masked-softmax, ReLU, ELU$+1$, and an RBF kernelization for $\tilde h$.
Figure~\ref{fig:nonlinear-td-verification} shows that $E_l$ stays at numerical precision across depth,
confirming that the residual update realizes a one-step nonlinear in-context TD update.
\begin{figure}[t]
  \centering
  \includegraphics[width=0.6\linewidth]{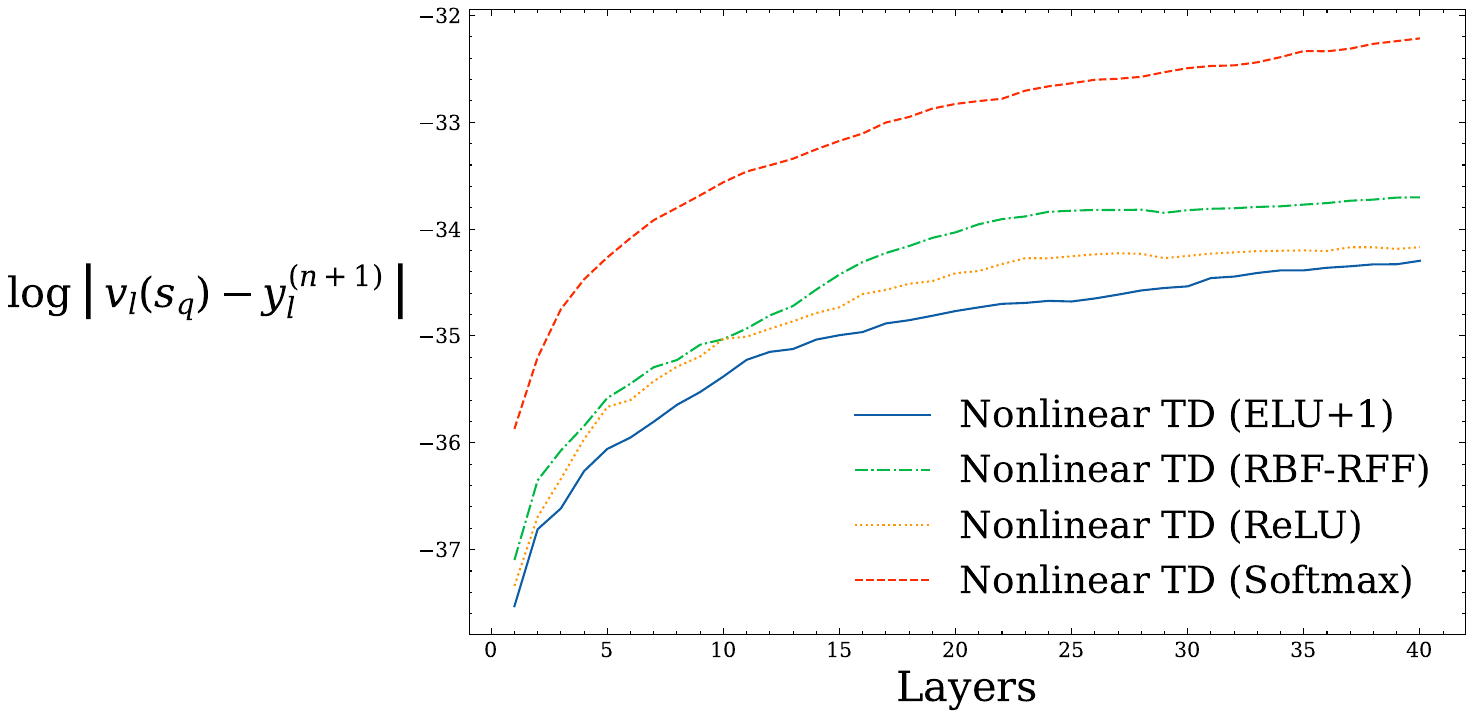}
  \caption{Kernel weighted TD verification. Layer-wise log discrepancy $E_l$ under different kernelizations $\tilde h$.}
  \label{fig:nonlinear-td-verification}
\end{figure}

\end{document}